\documentclass[11pt]{article}
\usepackage[margin=1.25in]{geometry}
\RequirePackage[OT1]{fontenc}
\RequirePackage{amsthm,amsmath,amsmath,amsfonts,amssymb}
\RequirePackage{natbib}
\RequirePackage[colorlinks,citecolor=blue,urlcolor=blue]{hyperref}

\RequirePackage{tablefootnote}

\RequirePackage{graphicx}

\usepackage{times}

\usepackage[boxed]{algorithm2e}
\SetKwRepeat{Do}{do}{while}%
\SetKwRepeat{Do}{do}{while}%
\SetKwInput{KwInput}{Input}
\SetKwInput{KwOutput}{Output}
\SetKwInput{KwFunction}{Function}

\usepackage{titling}
\usepackage{yfonts}

\usepackage{authblk}

\allowdisplaybreaks[4]

\newtheorem{definition}{Definition}

\newtheorem{theorem}{Theorem}
\newtheorem{lemma}{Lemma}
\newtheorem{corollary}{Corollary}

\newtheorem{proposition}{Proposition}

\theoremstyle{remark}
\newtheorem{remark}{Remark}
\newtheorem{example}{Example}

\newcommand{\mat}{\mathbf }
\newcommand{\vct}{\boldsymbol }

\newcommand{\ud}{\mathrm d}

\newcommand{\kl}{\mathrm{KL}}

\newcommand{\SPAN}{\mathrm{span}}

\newcommand{\SUM}{\mathrm{sum}}

\renewcommand{\tilde}{\widetilde}
\renewcommand{\hat}{\widehat}
\renewcommand{\bar}{\overline}

\def\mP{\mathbb P}

\def\sd{\mathfrak{d}}
\def\sp{\mathfrak{p}}
\def\sW{\mathsf{W}}
\def\ml{\mathsf{ML}}
\def\tv{\mathsf{TV}}

\def\op{\mathrm{op}}

\begin{document}

\title{Convergence Rates of Latent Topic Models Under Relaxed Identifiability Conditions}
\author[1]{Yining Wang}
\affil[1]{Machine Learning Department, Carnegie Mellon University}

\maketitle

\begin{abstract}
In this paper we study the frequentist convergence rate for the Latent Dirichlet Allocation \citep{blei2003latent} topic models.
We show that the maximum likelihood estimator converges to one of the finitely many equivalent parameters in Wasserstein's distance metric
at a rate of $n^{-1/4}$ without assuming separability or non-degeneracy of the underlying topics and/or the existence of more than three words per document,
thus generalizing the previous works of \cite{anandkumar2012spectral,anandkumar2014tensor} from an information-theoretical perspective.
We also show that the $n^{-1/4}$ convergence rate is optimal in the worst case.

\end{abstract}

\section{Introduction}

%We consider the classical \emph{Latent Dirichlet Allocation (LDA)} model for topic modeling of a collection of unlabeled documents \citep{blei2003latent}.
{
The \emph{Latent Dirichlet Allocation (LDA)} model, first introduced by \cite{blei2003latent}, has been very influential in machine learning as a probabilistic admixture model 
that characterizes latent topic structures in natural language document collections.
The original LDA paper \citep{blei2003latent} has accumulated a total of over 20,000 citations up to the year of 2017,
with many follow-up works also impactful in machine learning research \citep{griffiths2004finding,blei2012probabilistic,fei2005bayesian,blei2006dynamic}.
At a higher level, the LDA model posits the existence of $K$ latent (unknown) topic vectors,
and models the generation of a document as a collection of $m$ \emph{conditionally independent} words given a mixing topic vector for the document.

More specifically, let $V$ be the vocabulary size, $K$ be the number of topics 
and denote conveniently each of the $V$ words in the vocabulary as $1,2,\cdots,V$.
Let $\vct\theta=(\theta_1,\cdots,\theta_K)$ where $\theta_k\in\Delta^{V-1} = \{\pi\in\mathbb R^V: \pi\geq 0, \sum_i\pi_i=1\}$ be a collection of $K$ fixed but unknown topic word distribution vectors
that one wishes to estimate.
The LDA then models the generation of a document $X=(x_1,\cdots,x_m)\subseteq\{1,\cdots,V\}=:[V]$ of $m$ words as follows:
\begin{equation}
(x_1,\cdots,x_m)|h\overset{i.i.d.}{\sim}\text{Categorical}(h_1\theta_1+\cdots+h_K\theta_K),\;\;\;\;
h\sim\nu_0.
\label{eq:model}
\end{equation}
Here $\text{Categorical}(\pi)$ is the categorical distribution over $[V]$ parameterized by $\pi\in\Delta^{V-1}$, meaning that
$p(x=j|\pi) = \pi_j$ for $j\in[V]$, and $\nu_0$ is a known distribution that generates the ``mixing vector'' $h\in\Delta^{K-1}$.
A likelihood model $p_{\vct\theta}(x)$ can then be explicitly written out as
\begin{align}
p_{\vct\theta}(x) 
&= \int_{\Delta^{K-1}}p_{\vct\theta,h}(x)\ud\nu_0(h) = \int_{\Delta^{K-1}} \left[\prod_{i=1}^m p_{\vct\theta,h}(x_i)\right]\ud\nu_0(h)\nonumber\\
&= \int_{\Delta^{K-1}} \left[\prod_{i=1}^m\sum_{k=1}^Kh_k\theta_k(x_i)\right]\ud\nu_0(h)
\end{align}
for every $x=(x_1,\cdots,x_m)\in[V]^m$.
In the original LDA model \citep{blei2003latent} $\nu_0$ is taken to be the Dirichlet distribution,
while in this paper we allow $\nu_0$ to belong to a much wider family of distributions.
}

The objective of this paper is to study rates of convergence for estimating $\vct\theta$
from a collection of independently sampled unlabeled documents $X_1,\cdots,X_n$.
%without making separability conditions on $\vct\theta$.
Each document is assumed to be of the same length $m$.
\footnote{Our analysis is still valid if the length of each document is sandwiched between two constants. However we decide to proceed
with the assumption that each document is of equal length to simplify presentations.}
The estimation error between the underlying true model $\vct\theta$ and an estimator $\hat{\vct\theta}$ is evaluated by their \emph{Wasserstein's distance}:
\begin{equation}
d_\sW(\vct\theta,\hat{\vct\theta}) = \min_{\pi:[K]\to[K]} \sum_{k=1}^K{\|\theta_k-\hat\theta_{\pi(k)}\|_1},
\label{eq:wasserstein}
\end{equation}
where $\pi:[K]\to[K]$ is a \emph{permutation} on $K$.
When $K$ and $V$ are fixed, the $\ell_1$-norm in the definition of Eq.~(\ref{eq:wasserstein}) is not important as all vector $\ell_p$ norms are equivalent.

When $\vct\theta$ satisfies certain non-degenerate conditions, such as
 $\{\theta_j\}_{j=1}^K$ being linear independent \citep{anandkumar2012spectral,anandkumar2014tensor} or satisfying stronger ``anchor word'' \citep{arora2012learning} or ``$p$-separability'' 
conditions \citep{arora2013practical},
computationally tractable estimators exist that recover $\vct\theta$ at an $n^{-1/2}$ rate measured
in the Wasserstein's distance $d_{\sW}(\cdot,\cdot)$.
The general case of $\vct\theta$ being non-separable or degenerate, however, is much less understood.
To the best of our knowledge, the only convergence result for general $\vct\theta$ case in the $d_{\sW}(\vct{\hat\theta},\vct\theta)$
distance measure is due to \cite{nguyen2015posterior},
who established an $n^{-1/2(K+\alpha)}$ posterior contraction rate for hierarchical Dirichlet process models.
We discuss in Sec.~\ref{subsec:related} several important differences between \citep{nguyen2015posterior} and this paper.

We analyze the maximum likelihood estimation of the topic model in Eq.~(\ref{eq:model}) and show that, 
with a relaxed ``finite identifiability'' definition,
the ML estimator converges to one of the finitely many equivalent parameterizations (see Definition \ref{defn:finite-identifiability} and Theorem \ref{thm:main}
for a rigorous statement)
in Wasserstein's distance $d_{\sW}(\cdot,\cdot)$ at the rate of at least $n^{-1/4}$
even if $\{\theta_j\}_{j=1}^K$ are non-separable or degenerate.
Such rate is shown to be optimal by considering a simple ``over-fitting'' example.
In addition, when $\{\theta_j\}_{j=1}^K$ are assumed to be linear independent, we recover the $n^{-1/2}$ parametric convergence rate 
established in \citep{anandkumar2012spectral,anandkumar2014tensor}.

In terms of techniques, we adapt the classical analysis of rates of convergence for ML estimates in \citep{van1998asymptotic}
to give convergence rates under finite identifiability settings.
We also use Le Cam's method to prove corresponding local minimax lower bounds.
At the core of our analysis is a binomial expansion of the total-variation (TV) distance between distributions induced by neighboring parameters,
and careful calculations of the ``level of degeneracy'' in the TV-distance expansion of topic models,
which subsequently determines the convergence rate.

%{\bf [Remark that $m=1$ is not identifiable]}

\subsection{Related work}\label{subsec:related}

In the non-degenerate case where $\{\theta_j\}_{j=1}^K$ are linear independent, 
\cite{anandkumar2012spectral,anandkumar2014tensor,arora2012learning} applied the method of moments with noisy tensor decomposition techniques
to achieve the $n^{-1/2}$ parametric rate for recovering the underlying topic vectors $\vct\theta$ in Wasserstein's distance.
Extension and generalization of such methods are many, including supervised topic models \citep{wang2014spectral}, model selection \citep{cheng2015model},
computational efficiency \citep{wang2015fast} and online/streaming settings \citep{huang2015online,wang2016online}.
Under slightly stronger ``anchor word'' type assumptions, \cite{arora2012learning} developed algorithms beyond spectral decomposition of empirical tensors
and \cite{arora2013practical} demonstrated empirical success of the proposed algorithms.
%All of them require certain linear independence assumptions, which lead to parametric $n^{-1/2}$ rate.

Topic models are also intensively studied from a Bayesian perspective, with Dirichlet priors imposed on the underlying topic vectors $\vct\theta$.
Early works considered variational inference \citep{blei2003latent} and Gibbs sampling \citep{griffiths2004finding}
for generating samples or approximations of the posterior distribution of $\vct\theta$.
%\cite{nguyen2016borrowing,nguyen2015posterior,nguyen2013convergence} discuss rates of posterior contraction for a range of Dirichlet mixture models.
\cite{tang2014understanding,nguyen2015posterior} considered the posterior contraction of the convex hull of topic vectors and 
derived an $\tilde N^{-1/2}$ upper bound on the posterior contraction rate, where $\tilde N=\frac{\log n}{n} + \frac{\log m}{m} + \frac{\log m}{n}$.
\cite{nguyen2013convergence,nguyen2016borrowing} further considered the more difficult question of posterior contraction with respect to the Wasserstein's distance.
Apart from the Bayesian treatments of posterior contraction that contrasts our frequentist point of view of worst-case convergence, 
one important aspect of the work of \citep{tang2014understanding,nguyen2015posterior,nguyen2013convergence,nguyen2016borrowing}
is that the number of words per document $m$ has to grow together with the number of documents $n$,
and the posterior contraction rate becomes vacuous (i.e., constant level of error) for fixed $m$ settings.
In contrast, in this paper we consider $m$ being fixed as $n$ increases to infinity.
%this paper and \citep{nguyen2016borrowing,nguyen2015posterior,nguyen2013convergenc,e} 
%is that we consider 
%the number of words per document ($m$) to be fixed and let the number of documnets ($n$) grow, while the series works of Nguyen requires both $m$ and $n$ to grow,
%and the error bounds of estimation of $\vct\theta$ in Wasserstein's distance became vacuous for fixed $m$ settings in %\cite{nguyen2016borrowing,nguyen2015posterior,nguyen2013convergence}.
%In addition, when topics in $\vct\theta$ are not well-separated, the convergence rate in \cite{nguyen2016borrowing,nguyen2015posterior,nguyen2013convergence}
%falls back to an $n^{-1/2(K+\alpha)}$ rate and is slow when the number of topics $K$ is not too small.
%In contrast, in this paper we show that even if $\vct\theta$ is not well-separated, 
%the rate of $n^{-1/4}$ can always be achieved.
%the rate of $n^{-1/4}$ can be achieved under the additional condition
%that $\log\log V$ topics in $\vct\theta$ are well-separated.

Our work is also closely related to convergence analysis of \emph{singular} finite-mixture models.
In fact, our $n^{-1/4}$ convergence rate can be viewed as a ``discretized version'' of the seminal result of \cite{chen1995optimal},
who showed that an $n^{-1/4}$ rate is unavoidable to recover mean vectors in a degenerate Gaussian mixture model with respect to the Wasserstein's distance.
Difference exists, however, as topic models have a $K$-dimensional mixing vector $h$ for each observation and are therefore technically not finite mixture models.
\cite{ho2016singularity} proposed a general algebraic statistics framework for singular finite-mixture models,
and showed that the optimal convergence rate for skewed-normal mixtures is $n^{-1/12}$.
More generally, singular learning theory is studied in \citep{watanabe2009algebraic,watanabe2013widely},
and the algebraic structures of Gaussian mixture/graphical models and structural equation models are explored in \citep{leung2016identifiability,drton2011global,drton2016algebraic}.
%\cite{ho2016singularity} analyzed the rate of convergence for the maximum likelihood estimation (MLE) of finite mixture models.
%
%Their results do not trivially extend to Dirichlet mixtures, and computing the so-called ``level of singularity'' for topic models is also a very challenging task.

\subsection{Limitations and future directions}

We state some limitations of this work and bring up important future directions.
In this paper the vocabulary size $V$ and the number of topics $K$ are treated as fixed constants
and their dependency in the asymptotic convergence rate is omitted.
In practice, however, $V$ and $K$ could be large and understanding the (optimal) dependency of these parameters is important.
We consider this as a \emph{high-dimensional} version of the topic modeling problem, whose convergence rate remains largely unexplored in the literature.

Our results, similar to existing works of \cite{anandkumar2012spectral,anandkumar2014tensor}, are derived under a ``fixed $m$'' setting.
In fact, the convergence rates remain nearly unchanged by uniformly sampling 2 or 3 words per document,
and it is not clear how longer documents could help estimation of the underlying topic vectors under our framework.
In contrast, the posterior contraction results in \citep{tang2014understanding,nguyen2015posterior} are only valid under the ``$m$ increasing'' setting.
We conjecture that the actual behavior of the ML estimator should be a combination of both perspectives: $m\geq 2$ and $n\to\infty$ are sufficient for consistent estimation,
and $m$ growing with $n$ should deliver faster convergence rates.

Finally, the ML estimator for the topic modeling problem is well-known to be computationally challenging, and computationally tractable alternatives such as tensor decomposition and/or 
non-negative matrix factorization are usually employed.
In light of this paper, it is an interesting question to design \emph{computationally efficient} methods that attain the $n^{-1/4}$ convergence rate 
without assuming separability or non-degeneracy conditions on the underlying topic distribution vectors.

\subsection{Additional notations}

For two distributions $P$ and $Q$, we write $d_{\tv}(P;Q) = \frac{1}{2}\int |\ud P-\ud Q| = \sup_A |P(A)-Q(A)|$ as the total variation distance between $P$ and $Q$,
and $\kl(P\|Q) = \int\log\frac{\ud P}{\ud Q}\ud P$ as the Kullback-Leibler (KL) divergence between $P$ and $Q$.
For a sequence of random variables $\{A_n\}$, we write $A_n = O_\mP(a_n)$ if for any $\delta\in(0,1)$, there exists a constant $C>0$ such that
$\limsup_{n\to\infty} \Pr[|A_n/a_n|>C] \leq \delta$.

\section{Main results}

\subsection{Assumptions and finite identifiability}

We make the following regularity assumptions on $\vct\theta$ and $\nu_0$:
\begin{enumerate}
\item[(A1)] There exists constant $c_0>0$ such that $\theta_j(\ell)>c_0$ for all $j\in[K]$ and $\ell\in[V]$;
\item[(A2)] $\nu_0$ is exchangeable, meaning that $\nu_0(\mathcal A)=\nu_0(\pi(\mathcal A))$ for any permutation $\pi:[K]\to[K]$;
furthermore, $\mathbb E_{\nu_0}[h_1^2]>\mathbb E_{\nu_0}[h_1h_2]$ for $K\geq 2$ and $\mathbb E_{\nu_0}[h_1^3]+2\mathbb E_{\nu_0}[h_1h_2h_3] > 3\mathbb E_{\nu_0}[h_1^2h_2]$ for $K\geq 3$.
\end{enumerate}

{
Condition (A1) assumes that all topic vectors $\{\theta_j\}_{j=1}^K$ in the underlying parameter $\vct\theta$ lie on the interior of the $V$-dimensional probabilistic simplex $\Delta^{V-1}$.
%Condition 1 imposes a positive lower bound on every entry in the topic vectors $\{\theta_k\}_k$.
This is a technical condition, which can be viewed as an analogue of the ``support condition'' in classical analysis of MLE
where parameters in the considered parameter set $\Theta=\{\vct\theta\}$ give rises to the same support on observables.
If (A1) is violated, then different parameterization $\vct\theta$ might lead to different support of observables, 
posing technical difficulties for our analysis. 
More specifically, Proposition 4 will no longer hold as $p_{\vct\theta}(x)$ could be arbitrarily small.
We also remark that (A1) is a well-received technical condition in previous works \citep{nguyen2015posterior,tang2014understanding}
on convergence rates of admixture models.
%which was also assumed in previous work \citep{nguyen2015posterior,tang2014understanding}.
We use $\Theta_{c_0}$ to denote all parameters $\vct\theta$ that satisfies (A1).
}

The assumption (A2) only concerns the mixing distribution $\nu_0$ which is known a priori,
and is satisfied by ``typical'' priors of $h$, such as Dirichlet distributions and the ``finite mixture'' prior $p_{\nu_0}(h=e_k)=1/K$, $\forall k\in[K]$.

Suppose $X_1,\cdots,X_n\in[V]^m$ are $n$ documents i.i.d.~sampled from Model (\ref{eq:model}), each with $m$ words.
Let
\begin{equation}
p_{\vct\theta,m}(X_i) = \int_{\Delta^{K-1}} \prod_{j=1}^m{p_{\vct\theta,h}(X_{ij})}\ud\nu_0(h)
\label{eq:likelihood}
\end{equation}
be the likelihood of $X_i$ with respect to parameter $\vct\theta$, where $p_{\vct\theta,h}(x) = \sum_{j=1}^K{h_j\theta_j(x)}$.
Alternatively, we also write $p_{\vct\theta,m}(X_i) = \mathbb E_h[p_{\vct\theta,h}(x)]$ where $p_{\vct\theta,h}(x)=\prod_{j=1}^m{p_{\vct\theta,h}(X_{ij})}$.
In the classical theory of statistical estimation, one necessary condition to consistently estimate $\vct\theta$ from empirical observations $\{X_i\}_{i=1}^n$ is the \emph{identifiability} of $\vct\theta$,
loosely meaning that different parameter in the parameter space gives rises to different distributions on the observables.
%More specifically, we have the following definition:
\begin{definition}[exact/classical identifiability]
A distribution class $\{p_{\theta}\}_{\theta\in\Theta}$ is \emph{identifiable} with respect to $\Theta$ if for any $\theta,\theta'\in\Theta$,
$d_{\tv}(p_{\theta};p_{\theta'})=0$ implies $\theta=\theta'$.
\end{definition}

In the context of mixture models, the classical notion of identifiability is usually too strong to hold.
For example, in most cases $\theta_1,\cdots,\theta_K$ can only be estimated up to permutations, provided that $\nu_0$ is exchangeable.
This motivates us to consider a weaker notion of identifiability, which we term as ``finite identifiability'':
\begin{definition}[finite identifiability]
A distribution class $\{p_{\theta}\}_{\theta\in\Theta}$ is \emph{finitely identifiable} with respect to $\Theta$ if for any $\theta\in\Theta$, 
$|\{\theta'\in\Theta: d_{\tv}(p_{\theta};p_{\theta'})=0\}|<\infty$.
\label{defn:finite-identifiability}
\end{definition}
%Note that finite identifiability is trivial for finite parameter spaces $\Theta$.

Finite identifiability is weaker than the classical/exact notion of identifiability in the sense that 
two different parameterization $\theta,\theta'\in\Theta$ is allowed to have the same observable distributions (almost everywhere),
making them indistinguishable from any statistical procedures.
On the other hand, finite identifiability is sufficiently strong that non-trivial convergence can be studied for any infinite parameter space $\Theta$.
Below we give a few examples on finite identifiable or non-identifiable distribution classes.

\begin{example}
If $d_{\tv}(p_{\theta};p_{\theta'})=0$ implies $d_{\sW}(\theta,\theta')=0$ then $\{p_{\theta}\}$ is finitely identifiable.
This includes a wide range of convergence results for finite mixture models \citep{chen1995optimal,hsu2013learning,ge2015learning}, in which the underlying parameter $\vct\theta=(\theta_1,\cdots,\theta_K)$
can be consistently estimated up to permutations.
\end{example}

\begin{example}
The LDA model (\ref{eq:model}) with $K\geq 2$ topics and $m=1$ word per document is \emph{not} finitely identifiable,
because any parameterization $\vct\theta=(\theta_1,\cdots,\theta_K)$ with the same ``average'' word distribution $\bar\theta=\frac{1}{K}\sum_{k=1}^K{\theta_k}$
yields the same distribution of documents, and for any $\vct\theta$ there are infinitely many $\vct\theta'$ that matches exactly its average distribution $\bar\theta$.
\label{exmp:non-identifiable}
\end{example}

\begin{example}
Under (A1) and (A2), 
the LDA model (\ref{eq:model}) with $K\geq 2$ topics and $m\geq 2$ words per document is finitely identifiable.
We prove this as Lemma \ref{lem:finite-identifiable} in Sec.~\ref{sec:proofs}.
\end{example}

{
As a result of Example \ref{exmp:non-identifiable},
we further assume at least $m\geq 2$ words per document are present throughout this paper, which is necessary to guarantee finite identifiability
and therefore enables our discussion of convergence rates (to one of the finitely identifiable parameters).
Formally, we assume
\begin{enumerate}
\item[(A3)] $m\geq 2$.
\end{enumerate}
}

\subsection{Order of degeneracy}
%Alternatively, we also write $p_{\vct\theta}(X_i^{m'}) = \mathbb E_h[p_{\vct\theta,h}(X_i^{m'})]$.

In this section we introduce a concept which we name the \emph{order of degeneracy},
which is later used to characterize the optimal local convergence rates of latent topic models.

\begin{definition}[Order of degeneracy]
Let $\mathcal X=[V]$ be the vocabulary set and $\mu$ be the counting measure on $\mathcal X$. 
Let $\mathcal X^m=[V]^m$ be the product space of $\mathcal X$ and $\mu^m$ be the product measure of $\mu$.
For any $\vct\theta=(\theta_1,\cdots,\theta_K)\subseteq\Delta^{V-1}$ and $1\leq\sp\leq m$,
the \emph{$\sp$th-order degeneracy criterion} $\sd_{m,\sp}(\vct\theta)$ is defined as
\begin{multline}
\sd_{m,\sp}(\vct\theta) := 
\inf_{\|\vct\delta\|_1= 1, \sum_{\ell}\delta_j(\ell)=0}\\
\int_{\mathcal X^m}
 \bigg|\mathbb E_h p_{\vct\theta,h}(x)
\sum_{1\leq i_1<\cdots<i_\sp\leq m}\frac{\vct\delta_h(x_{i_1})\cdots\vct\delta_h(x_{i_\sp})}{p_{\vct\theta,h}(x_{i_1})\cdots p_{\vct\theta,h}(x_{i_\sp})}\bigg|\ud\mu_m(x),
\label{eq:od}
\end{multline}
where $\vct\delta=(\delta_1,\cdots,\delta_K)\in\mathbb R^V$, $\|\vct\delta\|_1 := \sum_{k=1}^K{\|\delta_k\|_1}$ and $\vct\delta_h(x) = \sum_{k=1}^K{h_k\delta_k(x)}$.
%Note that $\delta_k$ does \emph{not} need to be on the simplex $\Delta^{V-1}$.
\end{definition}

{
The definition of $\sd_{m,\sp}(\vct\theta)$ arises from a Taylor expansion of the likelihood function at neighboring parameters $p_{\vct\theta',m}(x)-p_{\vct\theta,m}(x)$,
which is given in Eq.~(\ref{eq:rp}).
While Eq.~(\ref{eq:od}) appears complicated, for the purpose of convergence rates it suffices to check whether $\sd_{m,\sp}(\vct\theta)>0$ or $\sd_{m,\sp}(\vct\theta)=0$,
and the exact values of $\sd_{m,\sp}(\vct\theta)$ are not important.
We thus define
\begin{equation}
\sp(m;\vct\theta) := \min\left\{\sp\in\mathbb Z^+: \sd_{m,\sp}(\vct\theta)>0\right\}
\end{equation}
as the smallest positive integer such that $\sd_{m,\sp}(\vct\theta) > 0$. (If $\sd_{m,\sp}(\vct\theta)=0$ for all $1\leq\sp\leq m$ then define $\sp(m;\vct\theta) := \infty$.) 
The quantity $\sp(m;\vct\theta)$ will be used exclusively in Theorem \ref{thm:main} in the next section, establishing upper and local lower bounds on the convergence rates of $\vct\theta$.
Intuitively, the smaller $\sp(m;\vct\theta)$ is, the faster an estimator converges to $\vct\theta$ (or one of its finite equivalents),
with the special case of $\sp(m;\vct\theta)=1$ corresponding to the classical $n^{-1/2}$ convergence rate for regular parametric models.

We next give some additional results regarding $\sp(m;\vct\theta)$.
We show that under assumptions (A1) through (A3), it always holds that $\sp(m;\vct\theta)\leq 2$
regardless of the number of words per document (provided that $m\geq 2$, i.e., (A3)) and the underlying parameter $\vct\theta$.
This is shown in Lemma \ref{lem:second-order}, which essentially implies finite identifiability and a general $n^{-1/4}$ convergence rate under (A1) through (A3) by Theorem \ref{thm:main}.
Furthermore, Lemma \ref{lem:first-order} shows that under additional linear independence conditions $\sp(m;\vct\theta) = 1$, yielding the classical $n^{-1/2}$ rate that is faster than $n^{-1/4}$
for general $\vct\theta$.
We also give examples for which $\sp(m;\vct\theta) > 1$, showing that the $\sp(m;\vct\theta)\leq 2$ result in Lemma \ref{lem:second-order} cannot be improved unconditionally.
Finally, we remark on how to computationally evaluate $\sp(m;\vct\theta)$, even when the true $\vct\theta$ is unknown and only an estimate $\hat{\vct\theta}$ is available.
}

%To fully understand the convergence rate of topic models using Theorem \ref{thm:main},
%It is important to understand the degeneracy structure $\sd_{m,\sp}(\vct\theta)$ for different parameter sub-classes.
%In the following sections we give some concrete results on the \emph{first-order} and \emph{second-order} degeneracy structures $\sd_{m,1}$ and $\sd_{m,2}$.
%Throughout we assume $K,m\geq 2$ and (A1), (A2) hold unless otherwise specified.

\subsubsection{First-order identifiability}

%We first state a sufficient condition for the topic model to be identifiable on the first order. %meaning that $\sd_{m,1}>0$ and an $n^{-1/2}$ convergence rate. 
{
When an underlying parameter $\vct\theta$ satisfies $\sp(m;\vct\theta) = 1$, we say it has \emph{first-order identifiability}. 
By Theorem \ref{thm:main}, first-order identifiability of $\vct\theta$ essentially implies a (local) convergence rate of $n^{-1/2}$,
which is similar to convergence rates in classical parametric models \citep{van1998asymptotic}.
The objective of this sub-section is to discuss scenarios under which first-order identifiability is present.
}

Our first lemma shows that, if at least $m\geq 3$ words per document are present and the underlying topic vectors $\{\theta_1,\cdots,\theta_K\}$
are \emph{linear independent}, then first-order identifiability is guaranteed.
\begin{lemma}
If $\{\theta_j\}_{j=1}^K$ are linear independent
then $\sd_{3,1}(\vct\theta)>0$.
\label{lem:first-order}
\end{lemma}
{
\begin{remark}
Lemma \ref{lem:first-order} implies that $\sp(3;\vct\theta) =1$ if $\vct\theta$ consists of linearly independent topics.
Furthermore, because $\sp(\cdot;\vct\theta)$ is a monotonic function in $m$ (see Corollary \ref{cor:monotonicity}), 
we have $\sp(m;\vct\theta) = 1$ for all $m\geq 3$.
\end{remark}
}

{
Lemma \ref{lem:first-order} is a simple consequence of the convergence results of \citep{anandkumar2012spectral,anandkumar2014tensor}
and the local minimax lower bounds established in Theorem \ref{thm:main} of this paper.
More specifically, \cite{anandkumar2012spectral,anandkumar2014tensor} explicitly constructed method-of-moments estimators that attain $n^{-1/2}$ convergence rate
for $m= 3$ and linearly independent $\vct\theta$, which would violate the local minimax lower bound in Theorem \ref{thm:main} if $\sp(3;\vct\theta)>1$.
A complete proof of Lemma \ref{lem:first-order} is given in Sec.~\ref{subsec:proof-first-order}.
}

%Note that with the conclusion of Lemma \ref{lem:first-order}, we have that $\sp^*=\sp(m;\vct\theta)=1$ if $m\geq 3$ and hence the ML estimator has an optimal $n^{-1/2}$ convergence rate.
%This essentially recovers the convergence result of \citep{anandkumar2012spectral,anandkumar2014tensor},
%albeit by a different estimator (MLE instead of method of moments).

Lemma \ref{lem:first-order}, as well as the results of \cite{anandkumar2012spectral,anandkumar2014tensor}, require two conditions: that $\{\theta_j\}_{j=1}^K$ being linear independent,
and that $m\geq 3$, meaning that there are at least 3 words per document.
It is an interesting question whether both conditions are necessary to ensure first-order identifiability.
We give partial answers to this question in the following two lemmas.

\begin{lemma}
If $\theta_j=\theta_k$ for some $j\neq k$ then $\sd_{m,1}(\vct\theta)=0$ for all $m\geq 2$.
\label{lem:overfit}
\end{lemma}

\begin{lemma}
Suppose $\{\theta_k\}_{k=1}^K$ are distinct.
Then $\sd_{2,1}(\vct\theta) = 0$ if and only if $K\geq 3$.
\label{lem:first-order-m2}
\end{lemma}

{
Lemma \ref{lem:overfit} shows that, if duplicates exist in the $K$ underlying topics then $\vct\theta$ cannot have first-order identifiability,
regardless of how many words are present in each document.
It is proved by a careful construction of $\vct\delta=(\delta_1,\cdots,\delta_K)$ such that the contribution of $\delta_j$ cancels out 
$\delta_k$ on all $x\in\mathcal X^m$, exploiting the condition that $\theta_j(v)=\theta_k(v)$ for all $v\in\mathcal X$.
A complete proof of Lemma \ref{lem:overfit} is given in Sec.~\ref{subsec:proof-overfit}.

Lemma \ref{lem:first-order-m2} studies the first-order identifiability of $\vct\theta$ from a different perspective.
The ``IF'' part of Lemma \ref{lem:first-order-m2} shows that, as long as $K\geq 3$ topics are present,
merely having $m=2$ words per document cannot lead to first-order identifiability.
We prove this by constructing the $\vct\delta=(\delta_1,\cdots,\delta_K)$ vectors as $\delta_1\propto \theta_2-\theta_3$,
$\delta_2\propto \theta_3-\theta_1$, $\delta_3 \propto \theta_1-\theta_2$ and showing that
$\delta_1,\delta_2,\delta_3$ cancel out each other if only $m=2$ words are present in each document.
On the other hand, the ``ONLY IF'' part of Lemma \ref{lem:first-order-m2} is more intriguing, which states that $m=2$ words per document is sufficient for first-order
identifiability if only two distinct topic vectors are to be estimated.
The proof of the only if part is however much more complicated, involving analytically verifying the full-rankness of a coefficient matrix.
A complete proof of Lemma \ref{lem:first-order-m2} is given in Sec.~\ref{subsec:proof-first-order-m2}.

While Lemmas \ref{lem:overfit} and \ref{lem:first-order-m2} combined show the necessity of $m\geq 3$ and additional non-degeneracy condition in Lemma \ref{lem:first-order},
 we remark that Lemmas \ref{lem:first-order}, \ref{lem:overfit} and \ref{lem:first-order-m2} do not cover all cases of $\vct\theta$
in the parameter space.
One notable exception is when $m\geq 3$, $K\geq 3$ and $\{\theta_k\}_{k=1}^K$ are distinct but not linearly independent,
for which none of the three lemmas apply and whether such parameterization satisfies first-order identifiability remains an open question.
Nevertheless, in Sec.~\ref{subsec:computation} we give a computational routine that determines whether $\sp(m;\vct\theta)=1$
or $\sp(m;\vct\theta)>1$ using any consistent estimates $\hat{\vct\theta}$ of $\vct\theta$,
which nicely complements the analytical results in Lemmas \ref{lem:first-order}, \ref{lem:overfit} and \ref{lem:first-order-m2}.
}

%Lemma \ref{lem:overfit} shows that a certain degree of separability on $\vct\theta$ is necessary to ensure first-order identifiability,
%and Lemma \ref{lem:first-order-m2} shows that the $m\geq 3$ condition is also necessary, unless there are only two topics present.
%The case where $\{\theta_j\}_{j=1}^K$ are distinct but linear dependent, however, remains open. 

\subsubsection{Second-order identifiability}

{
When an underlying parameter $\vct\theta$ satisfies $\sp(m;\vct\theta)\leq 2$, we say it has \emph{second-order identifiability}.
By definition, if $\vct\theta$ satisfies first-order identifiability then it also satisfies second-order identifiability,
but the reverse statement is generally not true.
Hence, second-order identifiability is weaker than its first-order counterparts, which also suggests potentially slower rates of convergence in parameter estimation.
}

%The following lemma shows that topic models are generally second-order identifiable,
%without any separability or non-degeneracy conditions imposed on $\vct\theta$.
In this section we show, perhaps surprisingly, that \emph{all} parameterization $\vct\theta$ have second-order identifiability
under (A1) through (A3).
{
\begin{lemma}
For all $\vct\theta$, $\sd(2,2)(\vct\theta)\geq c(\nu_0)/V^3K>0$,
where $c(\nu_0):=\mathbb E_{\nu_0}[h_1^2-h_1h_2]>0$ is a positive constant only depending on $\nu_0$.
%$\sd_{2,2}(\vct\theta)>0$ for all $\vct\theta$ satisfying (A1) and (A2).
\label{lem:second-order}
\end{lemma}

\begin{remark}
Lemma \ref{lem:second-order} implies that $\sp(2;\vct\theta)\leq 2$ for all $\vct\theta$ satisfying (A1) and (A2).
By monotonicity of $\sp(\cdot;\vct\theta)$ (see Corollary \ref{cor:monotonicity}),
we also have $\sp(m;\vct\theta)\leq 2$ for all $m\geq 2$.
\end{remark}

While appears surprising, the proof of Lemma \ref{lem:second-order} is actually quite simple.
The key observation is the existence of documents consisting of identical words (i.e., $x=(x_1,x_2)$ where $x_1=x_2$),
on which the $\vct\delta_h(x_1)\vct\delta_h(x_2)$ term becomes a square and equals zero only if $\vct\delta=0$.
A complete proof of Lemma \ref{lem:second-order} is given in Sec.~\ref{subsec:proof-second-order}.
}

Lemma \ref{lem:second-order} shows that, for any underlying parameter $\vct\theta$, if there are at least 2 words per document then $\sp(m;\vct\theta)\leq 2$.
This also suggests a general $n^{-1/4}$ convergence rate of an ML estimate of $\vct\theta$, by Theorem \ref{thm:main}.
This conclusion holds even for the ``over-complete'' setting $K\geq V$, under which existing works require particularly strong prior knowledge on $\vct\theta$
(e.g., $\{\theta_j\}_{j=1}^K$ being i.i.d.~sampled uniformly from the $V$-dimensional probabilistic simplex) for (computationally tractable) consistent estimation \citep{anandkumar2017analyzing,ma2016polynomial}.

{
\subsubsection{Numerical checking of $\sd_{m,\sp}(\vct\theta)>0$}\label{subsec:computation}

As we remarked in previous sections, Lemmas \ref{lem:first-order}, \ref{lem:overfit} and \ref{lem:first-order-m2} do not cover all 
cases, and there are parameters $\vct\theta$ whose order of degeneracy is not determined by the above lemmas.
In addition, in practical applications it might be desirable to compute the order of degeneracy with only an estimate $\hat{\vct\theta}$ of the underlying parameter $\vct\theta$.
In this section we present numerical procedures that decides whether $\sd_{m,\sp}(\vct\theta)>0$.
We also show that the calculation can be carried out on estimates $\hat{\vct\theta}$ and show its asymptotic consistency for the special case of $\sp=1$.

\begin{proposition}
For any $\vct\theta$, $\sd_{m,\sp}(\vct\theta)>0$ if and only if the following polynomial system in $\{\delta_{jk}\}$, $j\in[K]$, $k\in[V]$
does not have non-zero solutions:
\begin{align*}
\sum_{1\leq i_1<\cdots<i_{\sp}\leq m}\sum_{j_1,\cdots,j_\sp=1}^K\xi(\vct i,\vct j;\vct\theta,x)\prod_{\ell=1}^{\sp}\delta_{j_\ell,x_{i_\ell}} = 0,& & \forall x=(x_1,\cdots,x_m)\in[V]^m;\\
\sum_{k=1}^V\delta_{jk} = 0,& & \forall j\in[K].
\end{align*}
Here the coefficients $\xi(\vct i,\vct j;\vct\theta,x)$ is defined as
$$
\xi(\vct i,\vct j;\vct\theta,x) := \mathbb E_h\left[\prod_{i\notin \{i_1,\cdots,i_{\sp}\}} p_{\vct\theta,h}(x_i)\prod_{\ell=1}^\sp h_{j_\ell}\right].
$$
\label{prop:poly-condition}
\end{proposition}
\begin{proof}
Because $\mu_m$ in the definition of $\sd_{m,\sp}(\vct\theta)$ is a counting measure, $\sd_{m,\sp}(\vct\theta) = 0$ if and only if
all terms within the integral in Eq.~(\ref{eq:od}) are zero. This gives the proposition.
\end{proof}

\begin{table}[t]
\centering
\caption{\small Numerical estimations of the $\ell_1$-condition number $\kappa_1(\mat A(\vct\theta)) := \|\mat A(\vct\theta)\|_1\|\mat A(\vct\theta)^{-1}\|_1$
for different $V,K,m$ and $\vct\theta$.
%where $\|A\|_1 := \sup_{\|u\|_1=1}\|Au\|_1$.
The numerical estimation procedure of $\kappa_1(\mat A(\vct\theta))$ was given in \citep{hager1984condition}
and adopted in \textsc{Matlab}'s \texttt{condest} routine.
Each entry in the topic vectors are i.i.d.~generated from $U[0,1]$ and then normalized so that $\|\theta_k\|_1=1$ for all $k\in[K]$.}
\label{tab:results}
%Vocabulary size set to $V=10$.}
\vskip 0.1in
\begin{tabular}{lccccc}
\hline
& $V$& $K$& $M$& $\kappa_1(\mat A(\vct\theta))$& $\sp(m;\vct\theta)$\\
\hline
linear independent $\{\theta_k\}$& 10& 3& 3& $2.9\times 10^4$& $=1$\\
linear independent $\{\theta_k\}$& 10& 3& 2& $1.1\times 10^{19}$& $>1$\\
linear independent $\{\theta_k\}$& 10& 2& 2& $8.0\times 10^2$& $=1$\\
$\theta_1=\theta_2$& 10& 2& 3& $6.0\times 10^{17}$& $>1$\\
$\theta_1=\theta_2\neq\theta_3$& 10& 3& 4& $2.1\times 10^{18}$& $>1$\\
$\theta_3=0.5(\theta_1+\theta_2)$& 10& 3& 3& $9.7\times 10^4$& $=1$\\
$\theta_3=0.8\theta_1+0.2\theta_2$& 10& 3& 3& $4.7\times 10^5$& $=1$\\
\hline
\end{tabular}
\end{table}

With Proposition \ref{prop:poly-condition}, $\sp(m;\vct\theta)$ can be determined by enumerating from $\sp=1$ to $\sp=m$ and recording the smallest $\sp$ such that
$\sd_{m,\sp}(\vct\theta) > 0$,
which is the smallest $\sp$ such that the polynomial system in Proposition \ref{prop:poly-condition} does not have non-zero solutions.

The polynomial system in Proposition \ref{prop:poly-condition} has maximum degree $\sp$.
In principle, whether such a polynomial system admits non-zero solutions can be decided by converting the system
under the \emph{Gr\"{o}bner basis} and apply results from computational algebraic geometry.
Such an approach is however technically very complicated, and soon becomes computationally intractable when $V$ is large.

Fortunately, our result in Lemma \ref{lem:second-order} shows that $\sp(m;\vct\theta)\leq 2$ under very mild conditions.
More specifically, as long as each document consists of at least $m\geq 2$ words, the task of determining $\sp(m;\vct\theta)$ reduces to checking whether $\sd_{m,1}(\vct\theta)>0$
only, as $\sp(m;\vct\theta)\leq 2$ is always correct.
Furthermore, to decide whether $\sd_{m,1}(\vct\theta)>0$ the polynomial system in Proposition \ref{prop:poly-condition} reduces to a \emph{linear system},
whose existence of non-trivial solutions is easily determined by the \emph{rank} of its design matrix.
The following proposition formalizes the above discussion.
\begin{proposition}
$\sd_{m,1}(\vct\theta)>0$ if and only if the following linear system does not have non-zero solutions:
\begin{align*}
\sum_{i=1}^m\sum_{j=1}^K \xi(i,j;\vct\theta,x)\delta_{j,x_i} = 0,& & \forall x=(x_1,\cdots,x_m)\in[V]^m;\\
\sum_{k=1}^V\delta_{jk} = 0,& & \forall j\in[K];
\end{align*}
where $\xi(i,j;\vct\theta,x) := \mathbb E_h[h_j\prod_{i'\neq i}p_{\vct\theta,h}(x_{i'})]$.
\label{prop:rank-condition}
\end{proposition}
\begin{proof}
Immediately follows Proposition \ref{prop:poly-condition}.
\end{proof}

Proposition \ref{prop:rank-condition} constructs a \emph{linear} system with $VK$ variables and $(V^m+K)$ equations.
The existence of a non-trivial (non-zero) solution can be determined by explicitly constructing the $(V^m+K)\times VK$ matrix $\mat A$
in the equation $\mat A\mathrm{vec}(\{\delta_{k}\}) = \mat 0$ and checking whether $\mat A$ has full column rank.

We give in Table \ref{tab:results} some computational results of $\sp(m;\vct\theta)$ for some representative $\vct\theta$ settings.
Due to physical constraints of numerical precision, we use the $\ell_1$-condition number $\kappa_1(\mat A(\vct\theta))$
as an indication of whether $\mat A(\vct\theta)$ has full column rank, where a large condition number suggests that $\mat A(\vct\theta)$ is rank-deficient.
The first 6 lines in Table \ref{tab:results} verify our results in Lemmas \ref{lem:first-order}, \ref{lem:overfit} and \ref{lem:first-order-m2}.
The last 2 lines in Table \ref{tab:results} provide additional information regarding the first-order identifiability of linearly dependent but distinct topic vectors $\{\theta_k\}$.
They show that $\{\theta_k\}_k$ is first-order identifiable (i.e., $\sp(m;\vct\theta)=1$) even if $\{\theta_k\}_k$ are linear dependent,
provided that they are distinct and $m\geq 3$.
It remains an open question to formally establish such first-order identifiability for distinct but linear dependent topics.

In practice, the underlying $\vct\theta$ is unknown and only an estimate $\hat{\vct\theta}$ is available.
The following proposition shows that the procedure of checking whether $\sd_{m,\sp}(\vct\theta)>0$ remains valid asymptotically
if one replaces $\vct\theta$ with $\hat{\vct\theta}$.
\begin{proposition}
Let $\mat A(\vct\theta)$ and $\mat A(\hat{\vct\theta})$ be the $(V^m+K)\times VK$ matrices constructed using $\vct\theta$ and $\hat{\vct\theta}$, respectively.
Let $\sigma_{\min}(\mat A(\vct\theta))$ and $\sigma_{\min}(\mat A(\hat{\vct\theta}))$ be the smallest singular values of $\mat A(\vct\theta)$ and $\mat A(\hat{\vct\theta})$.
If $d_{\sW}(\vct\theta,\hat{\vct\theta})\overset{p}{\to} 0$ then $\sigma_{\min}(\mat A(\hat{\vct\theta})) \overset{p}{\to} \sigma_{\min}(\mat A(\vct\theta))$.
\end{proposition}
\begin{proof}
By Weyl's inequality we know that $|\sigma_{\min}(\mat A(\hat{\vct\theta}))- \sigma_{\min}(\mat A(\vct\theta))|\leq \|\mat A(\hat{\vct\theta})-\mat A(\vct\theta)\|_\op$.
It is easy to verify that $[\mat A(\hat{\vct\theta})]_{ij}\overset{p}{\to}[\mat A(\vct\theta)]_{ij}$ for all $i,j$ provided that $d_{\sW}(\vct\theta,\hat{\vct\theta})\overset{p}{\to} 0$,
because the coefficients are invariant under permutation $\pi:[K]\to[K]$ thanks to (A2).
We then have $\|\mat A(\hat{\vct\theta})-\mat A(\vct\theta)\|_\op\overset{p}{\to} 0$ because $\mat A(\cdot)$ are finite-dimensional matrices.
\end{proof}

Proposition shows that by substituting $\vct\theta$ with a consistent estimator $\hat{\vct\theta}$ in the construction of the $(V^m+K)\times VK$ coefficient matrix $\mat A$
and comparing the least singular value of $\mat A$ with a small number that slowly grows to zero, we can decide consistently whether $\sd_{m,1}(\vct\theta)>0$
using only $\hat{\vct\theta}$.
}

\subsection{MLE and its convergence rate}

{
In this section we formulate the ML estimator of $\vct\theta$ and derive its rates of convergence using $\sp(m;\vct\theta)$ defined in Eq.~(\ref{eq:rp}),
which reflects the order of degeneracy at $\vct\theta$.
We also prove \emph{local} minimax lower bounds showing that the MLE is optimal.
}

The \emph{maximum likelihood estimation} $\hat{\vct\theta}^{\ml}_{n,m}$ is defined as
\begin{align}
\hat{\vct\theta}^{\ml}_{n,m}\in\arg\max_{\vct\theta\in\Theta_{c_0}} \sum_{i=1}^n{\log p_{\vct\theta,m}(X_i)}, %= \arg\min_{\vct\theta\in\Delta^{V-1}} \sum_{i=1}^n \int_{\Delta^{K-1}} p_{\vct\theta,h}(X_i)\ud\nu_0(h),
\label{eq:mle}
\end{align}
where $p_{\vct\theta,h}$ is the likelihood function defined in Eq.~(\ref{eq:likelihood}).
It should also be noted that $\hat{\vct\theta}^{\ml}_{n,m}$ is constrained to the parameter set $\Theta_{c_0}$, which is assumed to be known a priori.

%We are now ready to state the main convergence theorem for the ML estimator.

\begin{theorem}
Fix $K\geq 2$, $m\geq 2$, $\vct\theta=(\theta_1,\cdots,\theta_K)\in\Theta_{c_0}$.
%and let $\nu_0$ be a known mixing distribution satisfying (A2).
Let $\widetilde\Theta_{c_0}(\vct\theta):=\{\tilde{\vct\theta'}\in\Theta_{c_0}: d_{\tv}(p_{\vct\theta,m};p_{\vct\theta,m'})=0\}$ be the equivalent parameter set with respect to $\vct\theta$,
which is finite thanks to Lemma \ref{lem:finite-identifiable}.
%Suppose $\vct\theta$ and $\nu_0$ satisfy (A1) and (A2).
%Suppose for some $2\leq m'\leq m$ there exists $\sp\leq m'$ such that $\sd_{m',\sp'}(\vct\theta)=0$ for all $1\leq\sp'<\sp$ and $\sd_{m',\sp}(\vct\theta)>0$
%(such $\sp$, if exists, must be unique).
Let $\sp(m;\vct\theta)$ be defined as in Eq.~(\ref{eq:rp}), and suppose $\sp(m;\vct\theta)<\infty$.
%For $2\leq m'\leq m$, let $\sp(m';\vct\theta)$ be the smallest integer $\sp$ such that $\sd_{m',\sp}(\vct\theta)>0$ and $\sd_{m',\sp'}(\vct\theta)=0$ for all $1\leq\sp'<\sp$.
%(If no such integer exists, define $\sp(m';\vct\theta):=\infty$.)
%Then the following holds:
\begin{enumerate}
\item \emph{(Global convergence rate of the MLE)}.
%Suppose $\sp^* := \min_{2\leq m'\leq m}\sp(m';\vct\theta)<\infty$.
%If $\sp(m;\vct\theta)<\infty$ then
%Suppose $\sp^*(m;\vct\theta) := \min_{2\leq m'\leq m}\sp(m';\vct\theta)<\infty$. Then
\begin{equation}
\min_{\tilde{\vct\theta}\in\widetilde\Theta_{c_0}(\vct\theta)} d_\sW(\tilde{\vct\theta},\hat{\vct\theta}^{\ml}_{n,m}) =  O_\mP(n^{-1/2\sp(m;\vct\theta)})
\label{eq:rate-upper}
\end{equation}
under $p_{\vct\theta,m}$ (or equivalently $p_{\tilde{\vct\theta},m}$),
where in $O_\mP(\cdot)$ we hide dependency on $\nu_0,m$ and $\vct\theta$;
\item \emph{(Local minimax rate)}.
%Suppose $\sp(m;\vct\theta)<\infty$.
Then there exists a constant $r_{\vct\theta}>0$ depending only on $\nu_0,m$ and $\vct\theta$ such that
%Let $\Theta_{n,c_0}(\vct\theta) := \{\vct\theta'\in\Theta_{c_0}: d_\sW(\vct\theta,\vct\theta')\leq n^{-1/2\sp(m;\vct\theta)}\}$ be a shrinking neighborhood of $\vct\theta$. Then
\begin{equation}
\inf_{\hat{\vct\theta}}\sup_{\vct\theta'\in\Theta_{n}(\vct\theta)} \mathbb E_{\vct\theta'}\left[d_\sW(\vct\theta',\hat{\vct\theta})\right] = \Omega(n^{-1/2\sp(m;\vct\theta)}),
\label{eq:rate-lower}
\end{equation}
where $\Theta_n(\vct\theta)$ is a shrinking neighborhood of $\vct\theta$ defined as $\{\vct\theta'\in\Theta_{c_0}: d_{\sW}(\vct\theta,\vct\theta')\leq r_{\vct\theta}\cdot  n^{-1/2\sp(m;\vct\theta)}\}$.
\end{enumerate}
\label{thm:main}
\end{theorem}

\begin{remark}
Our proof for the lower bound part of Theorem \ref{thm:main} actually proves the stronger statement that, for any $\vct\theta'\in\Theta_n(\vct\theta)$,
there exists constant $\tau>0$ such that no procedure can distinguish $\vct\theta$ and $\vct\theta'$ with success probability smaller than $\tau$, as $n\to\infty$.
Note that Eq.~(\ref{eq:rate-lower}) is a direct corollary of this testing lower bound by Markov's inequality.
%Two-point test lower bound. Used to establish $\sd_{3,1}(\vct\theta)>0$ for linear independent $\vct\theta$.
\label{rem:test}
\end{remark}

%\begin{remark}
%The lower bound in Theorem \ref{thm:main} does not necessarily match the upper bound, because $\sp^*\neq\sp(m;\vct\theta)$ in general.
%However, in two important special cases $\sp^*=\sp(m;\vct\theta)$ and hence matching bounds are proved:
%if $\{\theta_j\}_{j=1}^K$ is linear independent and $m\geq 3$, in which $\sp^*=\sp(m;\vct\theta)=1$ and an $n^{-1/2}$ convergence rate is optimal;
%or if $\theta_j=\theta_k$ for some $j\neq k$ and $m\geq 2$, in which $\sp^*=\sp(m;\vct\theta)=2$ and an $n^{-1/4}$ convergence rate is optimal.
%\end{remark}

{
Theorem \ref{thm:main} characterizes the convergence rates of MLE \emph{locally} at parameters $\vct\theta\in\Theta_{c_0}$,
with the convergence rates dependent on $\sp(m;\vct\theta)\in\mathbb N$.
While convergence rates depending on $\vct\theta$ might seem like a weak result, we argue that such convergence is probably the best one can hope for,
and the ``local convergence'' results still provide much valuable information about the statistical estimation problem of latent topic models.
In particular, we have the following observations:
\begin{enumerate}

\item It is arguable that convergence rates depending on the underlying parameter $\vct\theta$ (or its close neighborhoods) are the best one can hope for.
Because if the worst-case convergence rates are considered over all $\vct\theta\in\Theta_{c_0}$, by our Theorem 1 and Lemma 4 the only reasonable convergence rate is $n^{-1/4}$ which is slow;
on the other hand, by restricting ourselves to ``local'' convergence we can hope to get much faster rates like $n^{-1/2}$ for certain parameter settings;

\item By deriving $\vct\theta$-specific convergence rates, we obtain more information about the structure of the statistical estimation problem in latent topic models.
In particular, our results show that when topic vectors are linearly independent, the convergence rate is much faster than cases when duplicate topic vectors are present.
This is an interesting observation and is largely unknown in previous research on latent topic models.

\item One common difficulty with $\vct\theta$-specific rates is the challenge of deriving matching lower bounds, because if $\vct\theta$ is known the trivial estimator of outputting $\vct\theta$ always has zero measure.
We get around this issue by considering a ``close neighborhood'' of $\vct\theta$ and derive ``local'' minimax rates for any statistical procedure,
which match the convergence rates of the ML estimator.
Such a ``local analysis'' was used, for example in \cite{van1998asymptotic} to show the optimality of $I(\theta_0)^{-1}$ of MLE under classical settings.
Our analysis, on the other hand, focuses on ``local rates of convergence'' as the Fisher's information $I(\theta_0)$ in our model is not necessarily invertible,
under which case rates worse than $n^{-1/2}$ is unavoidable.

\end{enumerate}

The upper bound on convergence rates of MLE in Theorem \ref{thm:main} is proved by adapting the classical analysis of \citep{van1998asymptotic}
and considering higher order of binomial approximation depending on $\sp(m;\vct\theta)$.
The local minimax lower bound is proved by considering two hypothesis $\vct\theta,\vct\theta'$ and applying the Le Cam's inequality.
The $n^{-1/2\sp(m;\vct\theta)}$ term arises in the upper bound of TV-distance between distributions induced by $\vct\theta$ and $\vct\theta'$,
which is again bounded by higher-order binomial approximations.
The complete proof of Theorem \ref{thm:main} is given in Sec.~\ref{subsec:proof-main}.
}

\section{Proofs}\label{sec:proofs}

In this section we prove the main results of this paper.
To simplify presentation, we use $C>0$ to denote any constant that only depends on $V,K,m,\nu_0$ and $c_0$.
We also use $C_{\vct\theta}>0$ to denote constants that further depends on $\vct\theta\in\Theta_{c_0}$, the underlying parameter
that generates the observed documents.
Neither $C$ nor $C_{\vct\theta}$ will depend on the number of observations $n$.

Before proving the main theorem and subsequent results on concrete values of $\sd_{m,\sp}$,
we first prove a key lemma that connects the defined degeneracy criterion
with the total-variation (TV) distance between measures corresponding to neighboring parameters.
The finite identifiability of $\{p_{\vct\theta,m}\}$ can then be established as a corollary of Lemmas \ref{lem:main} and \ref{lem:second-order}.

\begin{lemma}
Suppose $\vct\theta\in\Theta_{c_0}$, $m\geq 2$ and $\sp(m;\vct\theta)<\infty$.
%Then there exists constants $\epsilon_0(\vct\theta),C_{\vct\theta}>0$ depending on $\nu_0$ and $\vct\theta$ such that for all $0<\epsilon\leq\epsilon_0(\vct\theta)$, 
Then for any $0<\epsilon\leq\epsilon_0<1/2$, 
\begin{equation}
\inf_{\epsilon\leq d_{\sW}(\vct\theta,\vct\theta')\leq\epsilon_0} d_{\tv}(p_{\vct\theta,m}; p_{\vct\theta',m}) \geq \left[\sd_{m,\sp(m;\vct\theta)}(\vct\theta) - \frac{V^m\epsilon_0}{1-\epsilon_0}\right]\cdot \epsilon^{\sp(m;\vct\theta)};
\label{eq:main-lower}
\end{equation}
\begin{equation}
\sup_{d_{\sW}(\vct\theta,\vct\theta')\leq\epsilon} d_{\tv}(p_{\vct\theta,m}; p_{\vct\theta',m}) \leq \frac{V^m}{1-\epsilon}\cdot \epsilon^{\sp(m;\vct\theta)}.
\label{eq:main-upper}
\end{equation}
\label{lem:main}
%Furthermore, if $\sp(m;\vct\theta)=\infty$ then $d_{\tv}(p_{\vct\theta,m};p_{\vct\theta',m})=0$ for all $\vct\theta'\in\Theta_{c_0}$.
\end{lemma}
\begin{proof}
We first prove Eq.~(\ref{eq:main-lower}).
%Because $\sd_{m,\sp}(\vct\theta)>0$, we know that $\sp'\leq\sp$.
%By assumption, we know that $\sp'\leq\sp$.
%By definition of $\sp(m;\vct\theta)$ and $\sd_{m,\sp}(\vct\theta)$, there exists $x=(x_1,\cdots,x_{m})\in\mathcal X^{m}$ such that
%\begin{equation}
%\inf_{\substack{\|\vct\delta\|_1=1\\\sum_{\ell=1}^V\delta_j(\ell)=0}}\bigg|\mathbb E_h p_{\vct\theta,h}(x)\sum_{1\leq i_1<\cdots<i_{\sp(m;\vct\theta)}\leq m}\frac{\vct\delta_h(%x_{i_1})\cdots\vct\delta_h(x_{i_{\sp(m;\vct\theta)}})}{p_{\vct\theta,h}(x_{i_1})\cdots p_{\vct\theta,h}(x_{i_{\sp(m;\vct\theta)}})}\bigg| > 0.
%\label{eq:sd-individual}
%\end{equation}
%holds for all $\|\vct\delta\|_1=1$ and $\sum_{\ell=1}^V{\delta_j(\ell)}=0$.
%%Now consider arbitrary $\vct\theta'$ that $\epsilon\leq d_{\sW}(\vct\theta,\vct\theta')\leq\epsilon_0$,
Let $\tilde{\vct\delta}=\vct\theta'_{\pi}-\vct\theta$ under appropriate permutation $\pi:[K]\to[K]$ such that $\tilde\epsilon := \|\tilde{\vct\delta}\|_1 = d_\sW(\vct\theta,\vct\theta')\in [\epsilon,\epsilon_0]$.
We then have (without loss of generality let $\pi(k)\equiv k$)
\begin{align}
p_{\vct\theta',m}(x)-p_{\vct\theta,m}(x)
&= \mathbb E_h\left[p_{\vct\theta',h}(x)-p_{\vct\theta,h}(x)\right]\nonumber\\
&= \mathbb E_h\left\{p_{\vct\theta,h}(x)\left[\frac{\prod_{i=1}^mp_{\vct\theta',h}(x_i)}{\prod_{i=1}^mp_{\vct\theta,h}(x_i)} - 1\right]\right\}\nonumber\\
&=  \mathbb E_h\left\{p_{\vct\theta,h}(x)\left[\prod_{i=1}^m\left(1+\frac{p_{\vct\theta',h}(x_i)-p_{\vct\theta,h}(x_i)}{p_{\vct\theta,h}(x_i)}\right) - 1\right]\right\}\nonumber\\
&= { \mathbb E_h\left\{p_{\vct\theta,h}(x)\left[\prod_{i=1}^m\left(1+\frac{\sum_{j=1}^Kh_j\theta'_j(x_i)-\sum_{j=1}^Kh_j\theta_j(x_i)}{p_{\vct\theta,h}(x_i)}\right) - 1\right]\right\}}\nonumber\\
&= \mathbb E_h\left\{p_{\vct\theta,h}(x)\left[\prod_{i=1}^m\left(1+\frac{\tilde{\vct\delta}_h(x_i)}{p_{\vct\theta,h}(x_i)}\right) - 1\right]\right\}\nonumber\\
&=: \sum_{\sp'=1}^m{r_{\sp'}(x)},\label{eq:rp}
\end{align}
where $\tilde{\vct\delta}_h(x_i) = \sum_{j=1}^k{h_j\delta_j(x_i)}$, $\delta_j(x_i)=\theta'_j(x_i)-\theta_j(x_i)$ and 
$$
r_{\sp'}(x) := \mathbb E_h\left\{p_{\vct\theta,h}(x) \sum_{1\leq i_1<\cdots<i_{\sp'}\leq m}\frac{\tilde{\vct\delta}_h(x_{i_1})\cdots\tilde{\vct\delta}_h(x_{i_{\sp'}})}{p_{\vct\theta,h}(x_{i_1})\cdots p_{\vct\theta,h}(x_{i_{\sp'}})}\right\}.
$$
By definition of $\sp(m;\vct\theta)$ and $\sd_{m,\sp}(\vct\theta)$ we know that $r_{\sp'}(x')=0$ for all $1\leq\sp'<\sp(m;\vct\theta)$ and $x'\in\mathcal X^m$; therefore
$$
\int_{\mathcal X^m}|r_{\sp'}(x)|\ud\mu_m(x) = 0.
$$
For $\sp'=\sp(m;\vct\theta)$,
integrating over all $x\in\mathcal X^m$ with respect to the counting measure we have
%$|r_{\sp'}(x)|\geq \underline C_{m}\cdot\tilde\epsilon^{\sp'}$ for $\sp'=\sp(m;\vct\theta)$
\begin{align*}
&\;\;\int_{\mathcal X^m}|r_{\sp'}(x)| \ud\mu_m(x)
= \int_{\mathcal X^m}\left|\mathbb E_h\left\{p_{\vct\theta,h}(x) \sum_{1\leq i_1<\cdots<i_{\sp'}\leq m}\frac{\tilde{\vct\delta}_h(x_{i_1})\cdots\tilde{\vct\delta}_h(x_{i_{\sp'}})}{p_{\vct\theta,h}(x_{i_1})\cdots p_{\vct\theta,h}(x_{i_{\sp'}})}\right\}\right|\ud\mu_m(x)\\
&= \|\tilde{\vct\delta}\|_1^{\sp'}\cdot\int_{\mathcal X^m}\left|\mathbb E_h\left\{p_{\vct\theta,h}(x) \sum_{1\leq i_1<\cdots<i_{\sp'}\leq m}\frac{\tilde{\vct\delta}_h(x_{i_1})\cdots\tilde{\vct\delta}_h(x_{i_{\sp'}})}{p_{\vct\theta,h}(x_{i_1})\|\vct\delta\|_1\cdots p_{\vct\theta,h}(x_{i_{\sp'}})\|\vct\delta\|_1}\right\}\right|\ud\mu_m(x)\\
&\geq \|\tilde{\vct\delta}\|_1^{\sp'}\cdot\inf_{\substack{\|\vct\delta\|_1=1\\\sum_{\ell=1}^V\delta_j(\ell)=0}}\int_{\mathcal X^m}\left|\mathbb E_h\left\{p_{\vct\theta,h}(x) \sum_{1\leq i_1<\cdots<i_{\sp'}\leq m}\frac{{\vct\delta}_h(x_{i_1})\cdots{\vct\delta}_h(x_{i_{\sp'}})}{p_{\vct\theta,h}(x_{i_1})\cdots p_{\vct\theta,h}(x_{i_{\sp'}})}\right\}\right|\ud\mu_m(x)\\
&= \|\tilde{\vct\delta}\|_1^{\sp'}\cdot \sd_{m,\sp}(\vct\theta) = \sd_{m,\sp'}(\vct\theta)[d_{\sW}(\vct\theta,\vct\theta')]^{\sp'}.
 %\geq \sd_{m,\sp}(\vct\theta)\cdot \epsilon^{\sp'}.
%&\geq  C_{\vct\theta}\epsilon^{\sp'};
\end{align*}
Here the third line holds because $\vct\delta := \tilde{\vct\delta}/\|\tilde{\vct\delta}\|_1$ satisfies $\|\vct\delta\|_1=0$ and $\sum_{\ell=1}^V\delta_j(\ell)=0$.
%and the fourth-line holds by consider $x\in \mathcal X^m$ that maximizes Eq.~(\ref{eq:sd-individual}).
For $\sp(m;\vct\theta)\leq\sp'\leq m$ and all $x'\in\mathcal X^m$, it holds that
\begin{align*}
\int_{\mathcal X^m}|r_{\sp'}(x')| \ud\mu_m(x)
&\leq \int_{\mathcal X^m} \left|\mathbb E_h\left\{\sum_{1\leq i_1<\cdots<i_{\sp'}\leq m}\prod_{i'\notin \{i_1,\cdots,i_{\sp'}\}}p_{\vct\theta,h}(x_{i'})\prod_{j=1}^{\sp'}\tilde{\vct\delta}_h(x_{i_j})\right\}\right|\ud\mu_m(x)\\
&\leq V^m\cdot \|\tilde{\vct\delta}\|_1^{\sp'} = V^m [d_{\sW}(\vct\theta,\vct\theta')]^{\sp'}.
%&= \int_{\mathcal X^m}|r_{\sp'}(x)|\ud\mu_m(x) \leq \overline C_{m}\cdot\tilde\epsilon^{\sp'}
\end{align*}
%where $0<\underline C_{m}\leq \overline C_{m}<\infty$ are constants that only depend on $V, K, \nu_0, m$ and $c_0$.
Subsequently, using the fact that $d_{\sW}(\vct\theta,\vct\theta')\in[\epsilon,\epsilon_0]$ and $\epsilon_0<1/2$, we have
\begin{align*}
d_{\tv}(p_{\vct\theta,m};d_{\vct\theta',m})
&= \int_{\mathcal X^m}\bigg|\sum_{\sp'=1}^mr_{\sp'}(x)\bigg|\ud\mu_m(x)\\
&\geq \int_{\mathcal X^m}|r_{\sp(m;\vct\theta)}(x)|\ud\mu_m(x) - \sum_{\sp'=\sp(m;\vct\theta)+1}^m\int_{\mathcal X^m}|r_{\sp'}(x)|\ud\mu_m(x)\\
&\geq \sd_{m,\sp(m;\vct\theta)}(\vct\theta)[d_{\sW}(\vct\theta,\vct\theta')]^{\sp(m;\vct\theta)} - V^m \cdot \sum_{\sp'=\sp(m;\vct\theta)+1}^m [d_{\sW}(\vct\theta,\vct\theta')]^{\sp'}\\
&\geq \sd_{m,\sp(m;\vct\theta)}(\vct\theta)[d_{\sW}(\vct\theta,\vct\theta')]^{\sp(m;\vct\theta)} - \frac{V^m}{1-d_{\sW}(\vct\theta,\vct\theta')}\cdot [d_{\sW}(\vct\theta,\vct\theta')]^{\sp(m;\vct\theta)+1}\\
&\geq \left[\sd_{m,\sp(m;\vct\theta)}(\vct\theta) - \frac{V^m\epsilon_0}{1-\epsilon_0}\right]\cdot \epsilon^{\sp(m;\vct\theta)}.
\end{align*}
%for $\epsilon_0(\vct\theta) := \min\{ C_{\vct\theta}/4, 1/2\}$ 
%and all $\vct\theta'\in\Theta_{c_0}$ such that $\epsilon\leq d_{\sW}(\vct\theta,\vct\theta')\leq\epsilon_0(\vct\theta)$, we have
%\begin{align*}
%d_{\tv}(p_{\vct\theta,m};p_{\vct\theta',m}) %\geq \tilde\epsilon^{-\sp(m;\vct\theta)}d_{\tv}(p_{\vct\theta,m};p_{\vct\theta',m})\\
%&\geq \left[ C_{\vct\theta}\tilde\epsilon^{\sp(m;\vct\theta)} - \sum_{\sp'=\sp(m;\vct\theta)+1}^{m}\tilde\epsilon^{\sp'}\right]\\
%
%&\geq ( C_{\vct\theta} - 2\epsilon_0(\vct\theta))\tilde\epsilon^{\sp(m;\vct\theta)}\geq  C_{\vct\theta}/2\cdot \epsilon^{\sp(m;\vct\theta)}.
%\end{align*}
%Note also that $\underline C_m$ does not depend on either $\vct\theta'$ or $\epsilon$, 
%implying that the above inequality holds uniformly for all $\vct\theta'$ and $\epsilon\to 0^+$.
%Eq.~(\ref{eq:main-lower}) is thus proved.

We next prove Eq.~(\ref{eq:main-upper}).
Let again $\tilde{\vct\delta} := \vct\theta'_{\pi}-\vct\theta$ and $\tilde\epsilon := \|\tilde{\vct\delta}\|_1\leq\epsilon$ for all $\vct\theta'\in\Theta_{c_0}$
such that $d_{\sW}(\vct\theta,\vct\theta')\leq\epsilon$.
Then %for any $\epsilon\leq\epsilon_0\leq1/2$, 
\begin{align*}
d_{\tv}(p_{\vct\theta,m};p_{\vct\theta',m}) %\leq \tilde\epsilon^{-\sp(m;\vct\theta)}d_{\tv}(p_{\vct\theta,m};p_{\vct\theta',m})\\
&\leq \sum_{\sp'=\sp(m;\vct\theta)}^m \int_{\mathcal X^m}|r_{\sp'}(x)|\ud \mu_m(x)\\
&\leq V^m\sum_{\sp'=\sp(m;\vct\theta)}^{m}\tilde\epsilon^{\sp'}\\
&\leq V^m\sum_{\sp'=\sp(m;\vct\theta)}^m\epsilon^{\sp'} \leq \frac{V^m}{1-\epsilon}\cdot \epsilon^{\sp(m;\vct\theta)}.
\end{align*}
%Note that $\overline C_m<\infty$ does not depend on $\vct\theta'$ or $\epsilon\to 0^+$.
%Eq.~(\ref{eq:main-upper}) is thus proved.

\end{proof}

%\begin{lemma}[Monotonicity of $\sd_{m,\sp}$]
%If $\sd_{m,\sp}(\vct\theta)>0$ then $\sd_{m',\sp}(\vct\theta)>0$ for all $m'\geq m$.
%\label{lem:monotonicity}
%\end{lemma}
%\begin{proof}
%Assume by way of contradiction that $\sd_{m',\sp}(\vct\theta)=0$.
%Applying Lemma \ref{lem:main} we have
%$$
%\liminf_{\epsilon\to 0} \inf_{d_{\sW}(\vct\theta,\vct\theta')\leq\epsilon} \epsilon^{-\sp}d_{\tv}(p_{\vct\theta,m'}, p_{\vct\theta',m'})  = 0.
%$$
%On the other hand, by data processing inequality we know that $d_{\tv}(p_{\vct\theta,m},p_{\vct\theta',m})\leq d_{\tv}(p_{\vct\theta,m'},p_{\vct\theta',m'})$.
%Therefore, 
%$$
%\liminf_{\epsilon\to 0} \inf_{d_{\sW}(\vct\theta,\vct\theta')\leq\epsilon} \epsilon^{-\sp}d_{\tv}(p_{\vct\theta,m}, p_{\vct\theta',m})  = 0.
%$$
%Invoking Lemma \ref{lem:main} again we have $\sd_{m,\sp}(\vct\theta)=0$, which is a contradiction.
%\end{proof}

\begin{lemma}[Finite identifiability of $\{p_{\vct\theta,m}\}$]
$\{p_{\vct\theta,m}\}_{\vct\theta\in\Theta_{c_0}}$ is finitely identifiable if $K\geq 2$ and $m\geq 2$.
\label{lem:finite-identifiable}
\end{lemma}
\begin{proof}
By data processing inequality we know that $d_{\tv}(p_{\vct\theta,m};p_{\vct\theta',m})\geq d_{\tv}(p_{\vct\theta,2};p_{\vct\theta',2})$
for $m\geq 2$.
Therefore, we only need to prove finite identifiability for $\{p_{\vct\theta,2}\}_{\vct\theta\in\Theta_{c_0}}$, i.e., $m=2$.

We first consider the case of $K=2$ and let $\vct\theta=(\theta_1,\theta_2)$ be the underlying topics.
Let $\vct\theta'=(\theta_1',\theta_2')$ be one of its equivalent parameterization such that $d_{\tv}(p_{\vct\theta,2};p_{\vct\theta',2})=0$.
By the data processing inequality, we must have $d_{\tv}(p_{\vct\theta,1};p_{\vct\theta,1})=0$ and therefore
$$
\mathbb E_{\nu_0}[h_1]\theta_1(x) + \mathbb E_{\nu_0}[h_2]\theta_2(x) = \mathbb E_{\nu_0}[h_1]\theta_1'(x) + \mathbb E_{\nu_0}[h_2]\theta_2'(x), \;\;\;\;\;\;\forall x\in\mathcal X.
$$
Because $\nu_0$ is exchangeable, the above identity implies that
\begin{equation}
\theta_1(x)+\theta_2(x) =\theta_1'(x)+\theta_2'(x) \;\;\;\;\;\;\forall x\in\mathcal X.
\label{eq:k2-m1}
\end{equation}
We now consider document $X=(x_1,x_2)$ consisting of identical words $x_1=x_2=x\in\mathcal X$.
Because $$
p_{\vct\theta}(X)=\mathbb E_{\nu_0}\left[(h_1\theta_1(x)+h_2\theta_2(x))^2\right] = p_{\vct\theta'}(X),
$$
using the exchangeability of $\nu_0$ we have that
\begin{multline*}
\mathbb E_{\nu_0}[h_1^2]\left[\theta_1(x)^2+\theta_2(x)^2\right] + 2\mathbb E_{\nu_0}[h_1h_2] \theta_1(x)\theta_2(x)\\
= \mathbb E_{\nu_0}[h_1^2]\left[\theta_1'(x)^2+\theta_2'(x)^2\right] + 2\mathbb E_{\nu_0}[h_1h_2] \theta_1'(x)\theta_2'(x), \;\;\;\;\;\;\forall x\in\mathcal X.
\end{multline*}
Subtracting $\mathbb E_{\nu_0}[h_1]^2(\theta_1(x)+\theta_2(x))^2$ on both sides of the above identity and invoking Eq.~(\ref{eq:k2-m1}) that $\theta_1(x)+\theta_2(x)=\theta_1'(x)+\theta_2'(x)$, we have
$$
2\mathbb E_{\nu_0}[h_1h_2-h_1^2]\theta_1(x)\theta_2(x) = 2\mathbb E_{\nu_0}[h_1h_2-h_1^2]\theta_1'(x)\theta_2'(x) \;\;\;\;\;\;\forall x\in\mathcal X.
$$
Because $\mathbb E_{\nu_0}[h_1h_2-h_1^2]>0$ thanks to assumption (A2), we have
\begin{equation}
\theta_1(x)\theta_2(x) = \theta_1'(x)\theta_2'(x) \;\;\;\;\;\;\forall x\in\mathcal X.
\label{eq:k2-m2}
\end{equation}

When $\vct\theta=(\theta_1,\theta_2)$ is fixed, Eqs.~(\ref{eq:k2-m1},\ref{eq:k2-m2}) form a quadratic system of $\theta_1'(x),\theta_2'(x)$ for every $x\in\mathcal X$,
which has at most two solutions.
Therefore, $|\{\vct\theta': d_{\tv}(p_{\vct\theta,2};p_{\vct\theta',2})=0\}| \leq 2^V < \infty$,
and the finite identifiability is proved.

We next consider the case of $K\geq 3$ and $m=2$. 
We know that $\sd_{2,1}(\vct\theta)=0$ and $\sd_{2,2}(\vct\theta)\geq c(\nu_0)/V^3K$ for all $\vct\theta\in\Theta_{c_0}$,
thanks to Lemmas \ref{lem:first-order-m2} and \ref{lem:second-order}.
By choosing $\epsilon_0 := c(\nu_0)/[2V^5K+c(\nu_0)]$, by Lemma \ref{lem:main} we have 
\begin{equation}
d_{\tv}(p_{\vct\theta,2};p_{\vct\theta',2}) > 0, \;\;\;\;\;\forall d_{\sW}(\vct\theta,\vct\theta')\leq\epsilon_0.
\label{eq:packing}
\end{equation}
For arbitrary $\vct\theta\in\Theta_{c_0}$ let $\tilde\Theta_{c_0}(\vct\theta) := \{\vct\theta'\in\Theta_{c_0}: d_{\tv}(p_{\vct\theta,2};p_{\vct\theta',2})=0\}$
be the set of all its equivalent parameterizations.
By Eq.~(\ref{eq:packing}), $\tilde\Theta_{c_0}$ forms a \emph{packing} of $\Theta_{c_0}$ with radius $\epsilon_0$ with respect to $d_{\sW}(\cdot,\cdot)$.
Because $\epsilon_0>0$ is a positive constant depending only on $\nu_0,V,K$ and $\Theta_{c_0}$ is compact,
we conclude that $|\tilde\Theta_{c_0}(\vct\theta)|<\infty$.
%Consider arbitrary $\vct\theta\in\Theta_{c_0}$ and let $\tilde\Theta_{c_0}(\vct\theta) := \{\vct\theta': d_{\tv}(p_{\vct\theta,2};p_{\vct\theta',2})=0\}$
%be the class of all parameters not identifiable from $\vct\theta$.
%By Lemma \ref{lem:main} and \ref{lem:second-order}, for any $\vct\theta'\in\tilde\Theta_{c_0}(\vct\theta)$ we have $\sp(2;\vct\theta')\leq 2$
%and therefore there exist $\epsilon_0(\vct\theta')>0$ such that for all $\vct\theta''\in\Theta_{c_0}$ with $0<d_\sW(\vct\theta',\vct\theta'')\leq \epsilon_0(\vct\theta')$,
%$d_{\tv}(p_{\vct\theta',2};p_{\vct\theta'',2}) > 0$.
%By Lemmas \ref{lem:main} and \ref{lem:second-order}, we know that for any $\vct\theta,\vct\theta'\in\Theta_{c_0}$ 
%with $0<d_{\sW}(\vct\theta,\vct\theta')\leq\epsilon_0$ it holds that $d_{\tv}(p_{\vct\theta,2};p_{\vct\theta',2})>0$,
%where $\epsilon_0$ is a constant that only depends on $V, K, \nu_0,c_0$ and $m'=2$.
%Corollary \ref{cor:finite-identifiable} is then proved by the compactness of $\Theta_{c_0}$.
\end{proof}

\begin{corollary}[Monotonicity of $\sp(m;\vct\theta)$]
$\sp(m';\vct\theta)\leq\sp(m;\vct\theta)$ for all $m'\geq m$.
\label{cor:monotonicity}
\end{corollary}
\begin{proof}
If $\sp(m;\vct\theta)=\infty$ then the inequality automatically holds.
Suppose $\sp(m;\vct\theta)=\sp$ and assume by way of contradiction that $\sp(m';\vct\theta)=\sp'>\sp$ for some $m'>m$.
Invoking Lemma \ref{lem:main} and the data processing inequality, we know that for all $0<\epsilon<1/4$, 
\begin{equation}
\sup_{d_{\sW}(\vct\theta,\vct\theta')\leq 2\epsilon}d_{\tv}(p_{\vct\theta,m};p_{\vct\theta',m})
\leq \sup_{d_{\sW}(\vct\theta,\vct\theta')\leq 2\epsilon}d_{\tv}(p_{\vct\theta,m'};p_{\vct\theta',m'})
\leq \frac{V^m 2^{\sp'}}{1-2\epsilon}\cdot \epsilon^{\sp'}.
\label{eq:mono-intermediate1}
\end{equation}
On the other hand, because $\sp(m;\vct\theta)=\sp$, we know that for all $0<\epsilon\leq\epsilon_0<1/2$, 
\begin{align}
\inf_{\epsilon\leq d_{\sW}(\vct\theta,\vct\theta')\leq\epsilon_0} d_{\tv}(p_{\vct\theta,m};p_{\vct\theta',m}) 
&\geq \left[\sd_{m,\sp}(\vct\theta)-\frac{V^m\epsilon_0}{1-\epsilon_0}\right]\cdot \epsilon^{\sp}\nonumber\\
&= \frac{1}{\epsilon^{\sp'-\sp}}\left[\sd_{m,\sp}(\vct\theta)-\frac{V^m\epsilon_0}{1-\epsilon_0}\right]\cdot \epsilon^{\sp'}.
\label{eq:mono-intermediate2}
\end{align}
Eqs.~(\ref{eq:mono-intermediate1}) and (\ref{eq:mono-intermediate2}) clearly contradict each other by considering $\epsilon_0>0$
moderately small such that $\sd_{m,\sp}(\vct\theta)\geq 2V^m\epsilon_0/(1-\epsilon_0)$,
and $\vct\theta'$ sufficiently close to $\vct\theta$ such that
such that $\epsilon\leq d_{\sW}(\vct\theta,\vct\theta')\leq 2\epsilon$ and letting $\epsilon\to 0^+$.
Thus, we conclude that $\sp(m';\vct\theta)\leq \sp(m;\vct\theta)$.
\end{proof}

By Lemma \ref{lem:second-order} and Corollary \ref{cor:monotonicity}, we immediately have the following claim:
\begin{corollary}[Finiteness of $\sp(m;\vct\theta)$]
For any $\vct\theta\in\Theta_{c_0}$ and $m\geq 2$, $\sp(m;\vct\theta)\leq 2$.
\end{corollary}

%\begin{corollary}
%For any $\vct\theta\in\Theta_{c_0}$ and $m\geq 2$, $\sp(m;\vct\theta)<\infty$.
%\label{cor:sp-finite}
%\end{corollary}
%\begin{proof}
%Let $\vct\theta'\in\Theta_{c_0}$ be an arbitrary parameter that differs from $\vct\theta$.
%If $\sp(m;\vct\theta)=\infty$, then we have that
%$$
%p_{\vct\theta,m}(x)-p_{\vct\theta',m}(x) = \sum_{\sp'=1}^m{r_{\sp'}(x)} = 0
%$$
%holds for all $x\in\mathcal X^m$, where $r_{\sp'}(x)$ is defined in the proof of Lemma \ref{lem:main}.
%Thus, $d_{\tv}(p_{\vct\theta,m};p_{\vct\theta',m})=0$.
%Because $\vct\theta'$ is arbitrary, this implies that $\{p_{\vct\theta,m}\}_{\vct\theta\in\Theta_{c_0}}$ is not finitely identifiable, which contradicts Corollary \ref{cor:finite-identifiable}.
%\end{proof}

\subsection{Proof of Theorem \ref{thm:main}}\label{subsec:proof-main}

We use a multi-point variant of the classical analysis of maximum likelihood \citep[Sec.~5.8]{van1998asymptotic} to establish the rate of convergence for MLE,
and Le Cam's method to prove corresponding (local) minimax lower bounds.

\noindent\emph{{Proof of upper bound}}.
Let $\vct\theta\in\Theta_{c_0}$ be the underlying parameter that generates the data.
Define 
$$
\widetilde{\Theta}_{c_0}(\vct\theta) := \left\{\tilde{\vct\theta}\in\Theta_{c_0}: d_{\tv}(p_{\vct\theta,m};p_{\tilde{\vct\theta},m})=0\right\}
$$
 as the set of its equivalent parameterizations, which is guaranteed to be finite thanks to Lemma \ref{lem:finite-identifiable}.
For $\epsilon>0$, define 
$$
\bar\Theta_{c_0,\epsilon}(\vct\theta) := \left\{\vct\theta'\in\Theta_{c_0}: d_{\sW}(\vct\theta',\tilde{\vct\theta})\geq\epsilon, \forall\tilde{\vct\theta}\in\widetilde{\Theta}_{c_0}(\vct\theta)\right\}
$$
as the set of all parameters that are at least $\epsilon$ away from any equivalent parameterization in $\tilde\Theta_{c_0}(\vct\theta)$ in Wasserstein's distance $d_{\sW}(\cdot,\cdot)$.
The following technical proposition and corollary shows that $d_{\tv}(p_{\vct\theta,m};p_{\vct\theta';m})$ is uniformly lower bounded from below for all $\vct\theta'\in\bar\Theta_{c_0,\epsilon}(\vct\theta)$.
\begin{proposition}
For every fixed $\vct\theta\in\Theta_{c_0}$,
$d_{\tv}(p_{\vct\theta,m};p_{\vct\theta',m})$ is continuous in $\vct\theta'$ with respect to $\|\cdot\|_2$, meaning that for every $\varepsilon>0$, there exists $\delta>0$ such that
$|d_{\tv}(p_{\vct\theta,m};p_{\vct\theta',m})-d_{\tv}(p_{\vct\theta,m};p_{\vct\theta'',m})| \leq \varepsilon$ for all $\vct\theta',\vct\theta''\in\Theta_{c_0}$ such that $\|\vct\theta'-\vct\theta''\|_2\leq\delta$,
where $\|\vct\theta'-\vct\theta''\|_2 := \sqrt{\sum_{i=1}^K\sum_{j=1}^V|\theta_i'(j)-\theta_i''(j)|^2}$.
\label{prop:continuous}
\end{proposition}
Proposition \ref{prop:continuous} can be easily proved by explicitly expanding the total variation between distributions parameterized by two parameters $\vct\theta,\vct\theta'\in\Theta_{c_0}$.
We give its complete proof in the appendix.
As a consequence of Proposition \ref{prop:continuous}, we have the following corollary:
\begin{corollary}
For any $0<\epsilon<1/2$, $\inf_{\vct\theta'\in\bar\Theta_{c_0,\epsilon}(\vct\theta)} d_{\tv}(p_{\vct\theta,m};p_{\vct\theta',m}) > 0$.
\label{cor:uniform-min}
\end{corollary}
\begin{proof}
We first show that $\bar\Theta_{c_0,\epsilon}(\vct\theta)$ is \emph{compact} under the general topology of $\mathbb R^{VK}$ by treating each $\vct\theta$ as a $VK$-dimensional vector.
$\bar\Theta_{c_0,\epsilon}(\vct\theta)$ is obviously bounded (with respect to $\|\cdot\|_2$), because $\Theta_{c_0}$ is bounded and $\bar\Theta_{c_0,\epsilon}(\vct\theta)\subseteq \Theta_{c_0}$.
In addition, $\bar\Theta_{c_0,\epsilon}$ can be written as
\begin{equation}
\bar\Theta_{c_0,\epsilon}(\vct\theta) = \bigcap_{\vct\theta'\in\tilde\Theta_{c_0}(\vct\theta)}  \Theta_{c_0} \backslash \left\{\vct\theta''\in\mathbb R^{VK}: \|\vct\theta''-\vct\theta'\|_2 < \epsilon \right\}.
\end{equation}
Note that we have replaced $d_\sW(\cdot,\cdot)$ with the $\|\cdot\|_2$ norm, which remains correct because all permutations of a parameterization $\vct\theta'\in\tilde\Theta_{c_0,\epsilon}(\vct\theta)$
are also contained in $\tilde\Theta_{c_0,\epsilon}(\vct\theta)$.
Because $\Theta_{c_0}$ is closed, $\{\vct\theta''\in\mathbb R^{VK}: \|\vct\theta''-\vct\theta'\|_2 < \epsilon \}$ is open, and any intersection of closed sets are closed,
we conclude that $\bar\Theta_{c_0,\epsilon}(\vct\theta)$ is closed.
Therefore $\bar\Theta_{c_0,\epsilon}(\vct\theta)$ is compact.
Also, because $\epsilon<1/2$, $\bar\Theta_{c_0,\epsilon}(\vct\theta)$ is clearly non-empty.
By the extreme value theorem
\footnote{For any compact set $K\subset\mathbb R^d$ and continuous function $f:K\to\mathbb R$, $f$ attains its minimum and maximum on $K$.}
 and the fact that $d_{\tv}(p_{\vct\theta,m};p_{\vct\theta',m})$ is continuous in $\vct\theta'$ with respect to $\|\cdot\|_2$ (Proposition \ref{prop:continuous}),
$d_{\tv}(p_{\vct\theta,m};p_{\vct\theta',m})$ attains its minimum on $\bar\Theta_{c_0,\epsilon}(\vct\theta)$.
The corollary is then proved by noting that $d_{\tv}(p_{\vct\theta,m};p_{\vct\theta',m})>0$ for all $\vct\theta'\in\bar\Theta_{c_0,\epsilon}(\vct\theta)$.
\end{proof}

For any $\vct\theta,\vct\theta'\in\Theta_{c_0}$, and $X_1,\cdots,X_n\in\mathcal X^m$ i.i.d.~sampled from the underlying distribution $p_{\vct\theta,m}$,
define the ``empirical KL-divergence'' $\widehat{\kl}(p_{\vct\theta,m}\|p_{\vct\theta',m})$ as
$$
\widehat{\kl}_n(p_{\vct\theta,m}\|p_{\vct\theta',m}) := \frac{1}{n}\sum_{i=1}^n{\log\frac{p_{\vct\theta,m}(X_i)}{p_{\vct\theta',m}(X_i)}}.
$$
By definition of the ML estimator, we know $\inf_{\tilde{\vct\theta}\in\widetilde\Theta_{c_0}(\vct\theta)}d_{\sW}(\hat{\vct\theta}_{n,m}^{\ml}, \tilde{\vct\theta})\leq\epsilon$
provided that
\begin{equation}
\widehat{\kl}_n(p_{\vct\theta,m}\|p_{\vct\theta',m}) > 0 \;\;\;\;\text{for all }\vct\theta'\in\bar\Theta_{c_0,\epsilon}(\vct\theta).
\label{eq:empirical-kl}
\end{equation}
Furthermore, we know that the ``population'' version of Eq.~(\ref{eq:empirical-kl}) must be correct:
$\inf_{\vct\theta'\in\bar\Theta_{c_0,\epsilon}(\vct\theta)}\kl(p_{\vct\theta,m}\|p_{\vct\theta',m})>0$,
because the KL-divergence is lower bounded by the total-variation distance, which is further uniformly bounded away from below by Corollary \ref{cor:uniform-min}.
%for all $\vct\theta'\in\bar\Theta_{c_0,\epsilon}(\vct\theta)$,
%because $\kl(P\|Q)=0$ implies $d_{\tv}(P;Q)=0$, and all $\tilde{\vct\theta}\in\Theta_{c_0}$ satisfying $d_{\tv}(p_{\vct\theta,m};p_{\tilde{\vct\theta},m})=0$
%are contained in $\widetilde\Theta_{c_0}(\vct\theta)$ and thus excluded from $\bar\Theta_{c_0,\epsilon}(\vct\theta)$ by definition.
Therefore, to prove convergence rate of the MLE it suffices to upper bound the perturbation between empirical and population KL-divergence
and lower bounds the population divergence for all ${\vct\theta}'\in\bar\Theta_{c_0,\epsilon}(\vct\theta)$.

We first consider the simpler task of bounding the perturbation between $\widehat{\kl}_n(p_{\vct\theta,m}\|p_{\vct\theta',m})$ and its population version
$\kl(p_{\vct\theta,m}\|p_{\vct\theta',m})$.
Note that $\widehat{\kl}_n(p_{\vct\theta,m}\|p_{\vct\theta',m})$ is a sample average of i.i.d.~random variables.
Using classical empirical process theory, we have the following lemma that bounds the uniform convergence of $\widehat{\kl}_n$ towards $\kl$;
its complete proof is given in the appendix.
\begin{lemma}
There exists $C_{\vct\theta}>0$ depending only on $\vct\theta,c_0,m,\nu_0$ such that 
$$
\mathbb E_{\vct\theta} \sup_{\vct\theta'\in\Theta_{c_0}} \frac{\big|\widehat{\kl}_n(p_{\vct\theta,m}\|p_{\vct\theta',m})-\kl(p_{\vct\theta,m}\|p_{\vct\theta',m})\big|}{\sqrt{\kl(p_{\vct\theta,m}\|p_{\vct\theta',m})}} \leq \frac{C_{\vct\theta}}{\sqrt{n}}.
$$
\label{lem:sup-convergence}
\end{lemma}
As a corollary, by Markov's inequality we know that for all $\delta\in(0,1)$, with probability $1-\delta$
\begin{equation*}
\widehat{\kl}_n(p_{\vct\theta,m}\|p_{\vct\theta',m}) \geq \kl(p_{\vct\theta,m}\|p_{\vct\theta',m}) - \sqrt{\kl(p_{\vct\theta,m}\|p_{\vct\theta',m})}\cdot \frac{C_{\vct\theta}}{\delta\sqrt{n}}, \;\;\;\;\forall\vct\theta'\in\Theta_{c_0}.
\end{equation*}
Subsequently, with probability $1-\delta$
\begin{equation}
\inf_{\vct\theta'\in\bar\Theta_{\epsilon,c_0}}\widehat{\kl}_n(p_{\vct\theta,m}\|p_{\vct\theta',m}) > 0 \;\;\Longleftarrow\;\;
\inf_{\vct\theta'\in\bar\Theta_{\epsilon,c_0}}\kl(p_{\vct\theta,m}\|p_{\vct\theta',m}) > \frac{C_{\vct\theta}^2}{\delta^2n}.
\label{eq:sup-convergence}
\end{equation}

We next establish a lower bound on $\kl(p_{\vct\theta,m}\|p_{\vct\theta',m})$ for all $\vct\theta'\in\bar\Theta_{c_0,\epsilon}(\vct\theta)$.
%Let $m'\leq m$ be the integer that gives rises to $\sp^* = \sp(m';\vct\theta)$.
By Pinsker's inequality, we have that for any $\vct\theta'\in\bar\Theta_{c_0,\epsilon}(\vct\theta)$,
$$
\kl(p_{\vct\theta,m}\|p_{\vct\theta',m})
\geq 2d_{\tv}^2(p_{\vct\theta,m};p_{\vct\theta',m}).
$$

Define
$$
\epsilon_0(\vct\theta) := \min\left\{\frac{1}{4},\frac{\sd_{m,\sp(m;\vct\theta)}(\vct\theta)}{2V^m + \sd_{m,\sp(m;\vct\theta)}(\vct\theta)}\right\}.
$$
Invoking Lemma \ref{lem:main} and noting that 
$$
\sd_{m,\sp(m;\vct\theta)}(\vct\theta) - \frac{V^m\epsilon_0(\vct\theta)}{1-\epsilon_0(\vct\theta)} \leq \frac{1}{2}\sd_{m,\sp(m;\vct\theta)}(\vct\theta),
$$
we have for all $0<\epsilon<\epsilon_0(\vct\theta)$ that
%Subsequently, invoking Lemma \ref{lem:main}, we have that for all $0<\epsilon\leq\epsilon_0$
$$
\inf_{\vct\theta'\in\bar\Theta_{c_0,\epsilon}(\vct\theta)\backslash\bar\Theta_{c_0,\epsilon_0(\vct\theta)}(\vct\theta)}
\kl(p_{\vct\theta,m}\|p_{\vct\theta',m}) \geq\frac{1}{2}[\sd_{m,\sp(m;\vct\theta)}(\vct\theta)]^2\cdot \epsilon^{2\sp(m;\vct\theta)} =: \gamma_{\vct\theta}\epsilon^{2\sp(m;\vct\theta)},
$$
where $\gamma_{\vct\theta}>0$ is a positive constant independent of $n$ or $\epsilon$.
%where $\epsilon_0>0$ is a constant defined in Lemma \ref{lem:main} that only depends on $K,V,\nu_0,c_0$ and $m'$.
Furthermore, by Corollary \ref{cor:uniform-min} and the Pinsker's inequality we know that $\inf_{\vct\theta'\in\bar\Theta_{c_0,\epsilon_0(\vct\theta)}(\vct\theta)}\kl(p_{\vct\theta,m}\|p_{\vct\theta',m}) > 0$. Because $\epsilon_0(\vct\theta)$ does not depend on $\epsilon$ or $n$, this infimum must be bounded away from below by a constant
depending only on $\vct\theta,c_0,\nu_0$ and $m$. Subsequently, for sufficiently small $\epsilon>0$ we have
% because $\bar\Theta_{c_0,\epsilon_0}(\vct\theta)$ is a subset of $\Theta_{c_0}$ that does not depend on $\epsilon$, 
%and that all $\tilde{\vct\theta}\in\Theta_{c_0}$ satisfying $d_{\tv}(p_{\vct\theta,m};p_{\tilde{\vct\theta},m})=0$ are included in $\widetilde{\Theta}_{c_0}(\vct\theta)$
%and thus excluded in $\bar\Theta_{c_0,\epsilon_0}(\vct\theta)$, we have that 
%$$
%\inf_{\vct\theta'\in\bar\Theta_{c_0,\epsilon}(\vct\theta)} \kl(p_{\vct\theta,m}\|p_{\vct\theta',m}) \geq \gamma'_{\vct\theta} > 0
%$$
%where $\gamma'_{\vct\theta}$ is a positive constant that does not depend on $\epsilon$. 
%This implies that for all sufficiently small $\epsilon>0$
\begin{equation}
\inf_{\vct\theta'\in\bar\Theta_{c_0,\epsilon}(\vct\theta)}
\kl(p_{\vct\theta,m}\|p_{\vct\theta',m}) \geq \gamma_{\vct\theta}'\epsilon^{2\sp(m;\vct\theta)},
\label{eq:kl-lowerbound}
\end{equation}
where $\gamma_{\vct\theta}'$ is a positive constant depending only on $\vct\theta,c_0,\nu_0$ and $m$.

Combining Eqs.~(\ref{eq:empirical-kl}), (\ref{eq:sup-convergence}) and (\ref{eq:kl-lowerbound}) with $\epsilon\asymp n^{-1/2\sp(m;\vct\theta)}$
we complete the proof of convergence rate of the ML estimator.

\noindent\emph{{Proof of lower bound}}.
Let $n$ be sufficiently large such that $r_{\vct\theta} n^{-1/2\sp(m;\vct\theta)} < 1/2$,
where $r_{\vct\theta}$ is the positive constant in the definition of $\Theta_n(\vct\theta)$ that is independent of $n$.
Invoking Lemma \ref{lem:main} we have that
\begin{equation}
\sup_{\vct\theta'\in\Theta_{n}(\vct\theta)} d_{\tv}(p_{\vct\theta,m};p_{\vct\theta',m}) \leq 2V^m\cdot r_{\vct\theta}^{\sp(m;\vct\theta)}n^{-1/2},
\label{eq:tv-upperbound}
\end{equation}
%where $C_{\vct\theta}>0$ is the constant in Eq.~(\ref{eq:main-lower}) and $r_{\vct\theta}$ is the constant in the definition of $\Theta_n(\vct\theta)$,
%both independent of $n$.
In addition, for all $\vct\theta,\vct\theta'\in\Theta_{c_0}$ the following proposition upper bounds their KL-divergence using TV distance:
\begin{proposition}
There exists a constant $C>0$ depending only on $V,K,\nu_0,c_0$ and $m$ such that, for all $\vct\theta,\vct\theta'\in\Theta_{c_0}$, 
$$
\kl(p_{\vct\theta,m}\|p_{\vct\theta',m}) \leq  C\cdot d_{\tv}^2(p_{\vct\theta,m};p_{\vct\theta',m}).
$$
\label{prop:pinsker-reverse}
\end{proposition}
At a higher level, Proposition \ref{prop:pinsker-reverse} can be viewed as an ``exact'' reverse of the Pinsker's inequality with matching upper and lower bounds for the KL divergence.
It is not generally valid for arbitrary distributions, but holds true for our particular model with $\vct\theta,\vct\theta'\in\Theta_{c_0}$
because both $p_{\vct\theta,m}$ and $p_{\vct\theta',m}$ are supported and bounded away from below on a finite set.
We give the complete proof of Proposition \ref{prop:pinsker-reverse} in the appendix.

Let $\vct\theta'$ be an arbitrary parameterization in $\Theta_{n}(\vct\theta)$,
and let $p_{\vct\theta,m}^{\otimes n}=p_{\vct\theta,m}\times\cdots\times p_{\vct\theta,m}$ be the $n$-times product measure of $p_{\vct\theta,m}$,
Using Eq.~(\ref{eq:tv-upperbound}), Proposition \ref{prop:pinsker-reverse} and the fact that the KL-divergence is additive for product measures, we have
$$
\kl(p_{\vct\theta,m}^{\otimes n}\|p_{\vct\theta',m}^{\otimes n})
\leq n\cdot\kl(p_{\vct\theta,m}\|p_{\vct\theta',m}) \leq 2V^m\cdot r_{\vct\theta}^{2\sp(m;\vct\theta)}.
$$
Subsequently, using Pinsker's inequality we have
$$
d_{\tv}(p_{\vct\theta,m}^{\otimes n};p_{\vct\theta',m}^{\otimes n}) \leq \sqrt{2V^m\cdot r_{\vct\theta}^{2\sp(m;\vct\theta)}}.
$$
By choosing $r_{\vct\theta} := [8V^m]^{-2\sp(m;\vct\theta)}$ we can upper bound the right-hand side of the above inequality by $1/2$.
Applying Le Cam's inequality we conclude that no statistical procedure can distinguish $\vct\theta$ from $\vct\theta'$ using $n$ observations with success probability higher than $3/4$.
The lower bound is thus proved by Markov's inequality.

\subsection{Proof of Lemma \ref{lem:first-order}}\label{subsec:proof-first-order}

This lemma is essentially a consequence of \citep{anandkumar2014tensor},
which developed a $\sqrt{n}$-consistent estimator for linear independent topics via the method of moments.
More specifically, the main result of \citep{anandkumar2014tensor} can be summarized by the following theorem:
\begin{theorem}
Suppose $2\leq K\leq V$, $m= 3$ and consider the parameter subclass $\Theta_{\sigma_0,c_0} := \{\vct\theta\in\Theta_{c_0}: \sigma_{\min}(\vct\theta)\geq\sigma_0\}$,
where $\sigma_{\min}(\vct\theta) := \inf_{\|w\|_2=1}\|\sum_{j=1}^K{w_j\theta_j}\|_2$ is the least singular value of the topics vectors,
and $\sigma_0>0$ is a positive constant.
Then there exists a (computationally tractable) estimator $\hat{\vct\theta}_n$ such that for all $\vct\theta\in\Theta_{\sigma_0,c_0}$, 
$$
d_{\sW}(\hat{\vct\theta}_n,{\vct\theta}) \leq C_{\sigma_0}\cdot O_\mP(n^{-1/2}),
$$
where $C_{\sigma_0}>0$ is a constant that only depends on $V,K,\nu_0$ and $\sigma_0$.
\label{thm:anandkumar}
\end{theorem}
We remark that the original paper of \citep{anandkumar2014tensor} only considered the case where $\nu_0$ is the Dirichlet distribution.
However, our assumption (A2) is sufficient for the success of their proposed algorithms and analysis.

Next consider any $\vct\theta\in\Theta_{c_0}$ such that $\{\theta_k\}_{k=1}^K$ are linear independent.
Define $\sigma_{\vct\theta} := \sigma_{\min}(\vct\theta)/2 > 0$.
The ``shrinking neighborhood'' $\Theta_n(\vct\theta)$ defined in Theorem \ref{thm:main} is then contained in $\Theta_{\sigma_{\vct\theta},c_0}$
for sufficiently large $n$.
Let $\vct\theta'\in\Theta_n(\vct\theta)\subseteq\Theta_{\sigma_{\vct\theta},c_0}$ be such that $d_{\sW}(\vct\theta,\vct\theta')=\Omega( n^{-1/2\sp(3;\vct\theta)})$.
If $\sp(3;\vct\theta)=1$ we already proved $\sd_{3,1}(\vct\theta)>0$.
On the other hand, if $\sp(3;\vct\theta)>1$ we know that $d_{\sW}(\vct\theta,\vct\theta')=\Omega(n^{-1/4})$.
By Theorem \ref{thm:anandkumar}, there exists a statistical procedure that can distinguish $\vct\theta$ from $\vct\theta'$ with success probability arbitrarily close to 1 for sufficiently large $n$,
which violates the lower bound in Theorem \ref{thm:main} (Remark \ref{rem:test}).
Thus, it is concluded that $\sp(3;\vct\theta)=1$ and therefore $\sd_{3,1}(\vct\theta)>0$.

\subsection{Proof of Lemma \ref{lem:overfit}}\label{subsec:proof-overfit}

Consider $\vct\delta=(\delta_1,\cdots,\delta_K)$ with $\delta_j=\frac{1}{4}(e_1-e_2)$, $\delta_k=\frac{1}{4}(e_2-e_1)$ and $\delta_\ell=0$ for all $\ell\neq j,k$,
where $e_1=(1,0,\cdots,0)$ and $e_2=(0,1,0,\cdots,0)$ are standard basis vectors in $\mathbb R^V$.
Clearly $\|\vct\delta\|_1=1$ and $\sum_{\ell=1}^V{\delta_j(\ell)}=0$ for all $j\in[K]$.
Define $p_{\vct\theta,h}(x_{-i}) := \prod_{j\neq i}p_{\vct\theta,h}(x_j)$.
We then have, for arbitrary $x=(x_1,\cdots,x_{m})\in\mathcal X^{m}$, 
\begin{multline*}
\left|\mathbb E_hp_{\vct\theta,h}(x)\sum_{i=1}^{m}{\frac{\vct\delta_h(x_i)}{p_{\vct\theta,h}(x_i)}}\right|
= \left|\sum_{i=1}^{m}\sum_{\ell=1}^k{\delta_\ell(x_i) \mathbb E_h\left[h_\ell p_{\vct\theta,h}(x_{-i})\right]}\right|\\
\leq \frac{1}{2}\sum_{i=1}^{m}{\vct 1_{[x_i\in\{1,2\}]}\big|\mathbb E_h[h_j p_{\vct\theta,h}(x_{-i})] - \mathbb E_h[h_kp_{\vct\theta,h}(x_{-i})]\big|}.
\end{multline*} 
Because $\nu_0$ is exchangeable and $\theta_j=\theta_k$, we have that $\mathbb E_h[h_jp_{\vct\theta,h}(x_{-i})]=\mathbb E_h[h_kp_{\vct\theta,h}(x_{-i})]$
for all $x_{-i}\in\mathcal X^{m-1}$.
Thus, $\sd_{m,1}(\vct\theta)=0$.

\subsection{Proof of Lemma \ref{lem:first-order-m2}}\label{subsec:proof-first-order-m2}

\noindent\emph{Proof of the ``IF'' part}. Let $\theta_1,\theta_2$ and $\theta_3$ be any three topic vectors in $\vct\theta$.
We assume $\theta_1,\theta_2,\theta_3$ are distinct, because otherwise $\sd_{2,1}(\vct\theta)=0$ is already implied by Lemma \ref{lem:overfit}.
Consider $\vct\delta=(\delta_1,\cdots,\delta_K)$ defined as
\begin{eqnarray*}
\delta_1 &:=& (\theta_2-\theta_3)/6;\\
\delta_2 &:=& (\theta_3-\theta_1)/6;\\
\delta_3 &:=& (\theta_1-\theta_2)/6;\\
\delta_k &:=& 0, \;\;\;\;\forall 3<k\leq K.
\end{eqnarray*}
It is easy to verify that $\|\vct\delta\|_1=1$ and $\sum_{\ell=1}^V{\delta_k(\ell)}=0$ for all $k\in[K]$.
We then have, for any $x=(x,y)\in\mathcal X^2$, 
\begin{equation}
\mathbb E_hp_{\vct\theta,h}(x,y)\left[\frac{\vct\delta_h(x)}{p_{\vct\theta,h}(x)} + \frac{\vct\delta_h(y)}{p_{\vct\theta,h}(y)}\right]
= \mathbb E_h[\vct\delta_h(x)p_{\vct\theta,h}(y)] + \mathbb E_h[\vct\delta_h(y)p_{\vct\theta,h}(x)]\label{eq:cancel-intermediate1}
\end{equation}
By definition of $\vct\delta$, we have that $6\vct\delta_h(x) = \theta_1(x)(h_3-h_2) + \theta_2(x)(h_1-h_3) + \theta_3(x)(h_2-h_1)$.
Define $\beta := (\mathbb E_{\nu_0}[h_1^2] - \mathbb E_{\nu_0}[h_1h_2])/6$.
We then have
\begin{multline}
\mathbb E_h[\vct\delta_h(x)p_{\vct\theta,h}(y)] \\
= \beta\theta_1(x)(\theta_3(y)-\theta_2(y)) 
+ \beta\theta_2(x)(\theta_1(y)-\theta_3(y)) + \beta\theta_3(x)(\theta_2(y)-\theta_1(y)).
\label{eq:cancel-intermediate2}
\end{multline}
Similarly, 
\begin{multline}
\mathbb E_h[\vct\delta_h(y)p_{\vct\theta,h}(x)]\\
 = \beta\theta_1(y)(\theta_3(x)-\theta_2(x))
+ \beta\theta_2(y)(\theta_1(x)-\theta_3(x)) + \beta\theta_3(y)(\theta_2(x)-\theta_1(x)).
\label{eq:cancel-intermediate3}
\end{multline}
Comparing Eqs.~(\ref{eq:cancel-intermediate2},\ref{eq:cancel-intermediate3}) we note that 
$$\mathbb E_h[\vct\delta_h(x)p_{\vct\theta,h}(y)]=-\mathbb E_h[\vct\delta_h(y)p_{\vct\theta,h}(x)]$$
for all $(x,y)\in\mathcal X^2$, which means that the right-hand side of Eq.~(\ref{eq:cancel-intermediate1}) is always 0.
Therefore, $\sd_{2,1}(\vct\theta)=0$.

\noindent\emph{Proof of the ``ONLY IF'' part}.
We show that if $K=2$ and $\theta_1\neq\theta_2$ then $\sd_{2,1}(\vct\theta)>0$.
Define $\beta := \mathbb E_{\nu_0}[h_1h_2]$ and $\gamma := \mathbb E_{\nu_0}[h_1^2]-\mathbb E_{\nu_0}[h_1h_2]$.
By (A2) we have that $\gamma > 0$.
We then have
\begin{align*}
\mathbb E[\vct\delta_h(x)p_{\vct\theta,h}(y)]
&= \mathbb E[(h_1\delta_1(x)+h_2\delta_2(x))(h_1\theta_1(y)+h_2\theta_2(y))]\\
&= \delta_1(x)[\beta\bar\theta(y)+\gamma\theta_1(y)]+\delta_2(x)[\beta\bar\theta(y)+\gamma\theta_2(y)],
\end{align*}
where $\bar\theta(y) := \theta_1(y)+\theta_2(y)$.
Similarly, 
$$
\mathbb E[\vct\delta_h(y)p_{\vct\theta,h}(x)] = \delta_1(y)[\beta\bar\theta(x)+\gamma\theta_2(x)] + \delta_2(y)[\beta\bar\theta(x)+\gamma\theta_2(x)].
$$
We can then simplify Eq.~(\ref{eq:cancel-intermediate1}) as
\begin{align*}
T_{\vct\theta,x,y}(\vct\delta) 
&:= \mathbb E_hp_{\vct\theta,h}(x,y)\left[\frac{\vct\delta_h(x)}{p_{\vct\theta,h}(x)} + \frac{\vct\delta_h(y)}{p_{\vct\theta,h}(y)}\right]\\
&= \delta_1(x)[\beta\bar\theta(y)+\gamma\theta_1(y)]+\delta_2(x)[\beta\bar\theta(y)+\gamma\theta_2(y)]\\
&+ \delta_1(y)[\beta\bar\theta(x)+\gamma\theta_1(x)]+\delta_2(y)[\beta\bar\theta(x)+\gamma\theta_2(x)].
\end{align*}
%
%To incorporate the constraint $\sum_{\ell=1}^V{\delta_j(\ell)}=0$, we further define
%$$
%\widetilde T_{\vct\theta,j}(\vct\delta) := \sum_{\ell=1}^V{\delta_j(\ell)}, \;\;\;\;j=1,2.
%$$
%We then have the following equivalent charaterization for $\sd_{2,1}(\vct\theta)$:
%\begin{align*}
%\sd_{2,1}(\vct\theta) > 0 \;\;\Longleftrightarrow \;\; 
%&\left\{T_{\vct\theta,x,y}(\vct\delta) =0, \forall x,y\in[V]\right\}\text{has no non-trivial (non-zero) solutions $\vct\delta$.}
%\end{align*}

%Note that the system $T_{\vct\theta,x,y}(\vct\delta)=0$ and $\widetilde T_{\vct\theta,j}(\vct\delta)=0$ is a linear system
%in the concatenated vectorized form $\VEC(\vct\delta) = (\delta_1,\delta_2)\in\mathbb R^{2V}$ of $\vct\delta$.
%More specifically, the system consists of $2V$ variables and $V^2+2$ equations, and can be compactly written as

Assume by way of contradiction that $\sd_{2,1}(\vct\theta)=0$, which implies the existence of $\vct\delta\neq 0$, $\sum_{\ell=1}^V{\delta_j(\ell)}=0$ such that
$T_{\vct\theta,x,y}(\vct\delta)=0$ for all $x,y\in[V]$.
We then have
\begin{equation}
\mat B_1\delta_1 + \mat B_2\delta_2=0,
\label{eq:linear-system}
\end{equation}
where $\mat B_1=(b_{11},\cdots,b_{1V})$ and $\mat B_2=(b_{21},\cdots,b_{2V})$ are $K\times (V^2+2)$ matrices.
Furthermore, $b_{j\ell}$ for $j\in\{1,2\}$ and $\ell\in[V]$ can be explicitly formed as
\begin{align*}
b_{j\ell} = (\beta\bar\theta + \gamma\theta_j)(\vct e_{\ell\cdot}+\vct e_{\cdot\ell}) + \mu_{j\ell}\vct e_{\ell\ell}
\end{align*}
where $\mu_{j\ell} = \beta\bar\theta(\ell)+\gamma\theta_j(\ell)$ and
$\{\vct e_{\ell\ell'}\}_{\ell,\ell'=1}^V$ denotes the $V^2$ components of $b_{j\ell}$.
Subsequently,
\begin{align*}
\mat B_1\delta_1 + \mat B_2\delta_2
&= \sum_{\ell=1}^V{\delta_1(\ell)b_{1\ell}} + \sum_{\ell=1}^V{\delta_2(\ell)b_{2\ell}}\\
&=  \sum_{\ell=1}^V{\vct e_{\ell\cdot}\left[\sum_{j=1,2}\delta_j(\ell)(\beta\bar\theta+\gamma\theta_j) + \mu_{j\ell}\delta_j\right]};
\end{align*}
therefore,
\begin{equation}
\sum_{j=1,2}\delta_j(\ell)(\beta\bar\theta+\gamma\theta_j) + \mu_{j\ell}\delta_j = 0, \;\;\;\;\;\forall\ell\in[V].
\label{eq:cancel-intermediate4}
\end{equation}

We next state a technical proposition that will be proved in the appendix,
which shows that $\delta_1$ and $\delta_2$ can be expressed as linear combinations of $\beta\bar\theta+\gamma\theta_1$ and $\beta\bar\theta+\gamma\theta_2$:
\begin{proposition}
There exists $\xi_{11},\xi_{12},\xi_{21},\xi_{22}\in\mathbb R$ such that $\delta_1=\xi_{11}(\beta\bar\theta+\gamma\theta_1)+\xi_{12}(\beta\bar\theta+\gamma\theta_2)$ and
$\delta_2 = \xi_{21}(\beta\bar\theta_1+\gamma\theta_2)+\xi_{22}(\beta\bar\theta_1+\gamma\theta_2)$.
\label{prop:xi}
\end{proposition}
Substituting the expression of $\delta_1$ and $\delta_2$ in Proposition \ref{prop:xi} into Eq.~(\ref{eq:cancel-intermediate4}), we have
\begin{equation}
\sum_{j=1,2}(\beta\bar\theta+\gamma\theta_j)\left[\sum_{k=1,2}\mu_{k\ell}(\xi_{jk}+\xi_{kj})\right] = 0, \;\;\;\;\;\forall \ell\in[V].
\label{eq:cancel-intermediate5}
\end{equation}
Because $\beta\bar\theta+\gamma\theta_1$ and $\beta\bar\theta+\gamma\theta_2$ are linear independent if $\gamma>0$ and $\theta_1\neq\theta_2$,
it must hold that $\sum_{k=1,2}\mu_{k\ell}(\xi_{jk}+\xi_{kj})=0$ for all $j=1,2$ and $\ell\in[V]$.
Recall that $\mu_{k\ell} = \beta\bar\theta(\ell)+\gamma\theta_k(\ell)$. Subsequently, for $j=1,2$ we have
$$
\sum_{k=1,2}(\xi_{jk}+\xi_{kj})(\beta\bar\theta+\gamma\theta_k) = 0.
$$
Using again the fact that $\beta\bar\theta+\gamma\theta_1$ and $\beta\bar\theta+\gamma\theta_2$ are linear independent,
we conclude $\xi_{jk}+\xi_{kj}=0$ for all $k=1,2$.
Thus, $\xi_{11}=\xi_{22}=0$ and $\xi_{12}=-\xi_{21}$.
On the other hand, because $\SUM(\delta_1)=\SUM(\delta_2)=0$ and $\SUM(\beta\bar\theta+\gamma\theta_1)=\SUM(\beta\bar\theta+\gamma\theta_2)=\beta+\gamma>0$,
where $\SUM(z):=\sum_{\ell=1}^V{z(\ell)}$,
we must have $\xi_{11}+\xi_{12}=\xi_{21}+\xi_{22}=0$, and hence Eq.~(\ref{eq:linear-system}) only has the trivial solution $\delta_1=\delta_2=0$.
Thus, $\sd_{2,1}(\vct\theta)=0$.

\subsection{Proof of Lemma \ref{lem:second-order}}\label{subsec:proof-second-order}

For any $\ell\in[V]$ consider $x=(x_1,x_2)\in[V]^2$ where $x_1=x_2=\ell$.
Because of (A1), $p_{\vct\theta,h}(x)>0$ for all $h\in\Delta^{K-1}$.
Subsequently, 
\begin{align*}
\mathbb E_hp_{\vct\theta,h}(x)\frac{\vct\delta_h(x_1)\vct\delta_h(x_2)}{p_{\vct\theta,h}(x_1)p_{\vct\theta,h}(x_2)}
&= \mathbb E_h \left[\vct\delta_h(x_1)\vct\delta_h(x_2)\right] 
= \mathbb E_h\left[\left(\sum_{j=1}^k{h_j\delta_j(\ell)}\right)^2\right]\\
&= \mathbb E[h_1h_2] \left(\sum_{j=1}^k{\delta_j(\ell)}\right)^2 
+ (\mathbb E[h_1^2]-\mathbb E[h_1h_2]) \sum_{j=1}^k{\delta_j(\ell)^2}\\
&\geq c(\nu_0)\sum_{j=1}^k\delta_j(\ell)^2.
\end{align*}
Here in the last line we use the fact that $\nu_0$ is exchangeable and the definition that $c(\nu_0) = \mathbb E_{\nu_0}[h_1^2-h_1h_2] > 0$.
Subsequently, for every $\vct\delta$ satisfying $\|\vct\delta\|_1=1$, it holds that
\begin{align*}
\sd_{2,2}(\vct\theta)
&\geq V^{-2}\sum_{\ell=1}^V\left|\mathbb E_hp_{\vct\theta,h}(\ell,\ell)\frac{\vct\delta_h(\ell)^2}{p_{\vct\theta,h}(\ell)^2}\right|\\
&\geq V^{-2}\cdot c(\nu_0)\sum_{j=1}^k\sum_{\ell=1}^V\delta_j(k)^2\\
&\geq V^{-2}\cdot c(\nu_0)\cdot \frac{(\sum_{j=1}^k\sum_{\ell=1}^V |\delta_j(k)|)^2}{VK}
= \frac{c(\nu_0)\|\vct\delta\|_1^2}{V^3 K} = \frac{c(\nu_0)}{V^3 K}.
\end{align*}
%Since $\mathbb E[h_1^2]>\mathbb E[h_1h_2]$, the right-hand side of the above equation is zero if and only if $\delta_j(\ell)=0$ for all $j\in[K]$.
%Subsequently, because of the arbitrarity of $\ell\in [V]$, $\sd_{2,2}(\vct\theta)=0$ only if $\vct\delta=\{0,\cdots,0\}$, which contradicts $\|\vct\delta\|_1=1$.

\appendix

\section{Missing proofs}

We give missing proofs of technical lemmas in this appendix.

\subsection{Proof of Lemma \ref{lem:sup-convergence}}

For any $\vct\theta'\in\Theta_{c_0}$ define a $V^m$-dimensional random vector $v_{\vct\theta'}$ as
$v_{\vct\theta'}(x) := \log\frac{p_{\vct\theta,m}(x)}{p_{\vct\theta',m}(x)}$ for $x\in[V]^m$.
We then have that $\widehat{\kl}_n(p_{\vct\theta,m}\|p_{\vct\theta',m}) = \frac{1}{n}\sum_{i=1}^n{v_{\vct\theta'}(X_i)}$
and $\kl(p_{\vct\theta,m}\|p_{\vct\theta',m}) = \mathbb E_{\vct\theta}[v_{\vct\theta'}(X)]$.
By a simple re-scaling argument, we have that
\begin{equation}
\mathbb E_{\vct\theta}\sup_{\vct\theta'\in\Theta_{c_0}}\frac{|\widehat{\kl}_n(p_{\vct\theta,m}\|p_{\vct\theta',m})-\kl(p_{\vct\theta,m}\|p_{\vct\theta',m})|}{\|v_{\vct\theta'}\|_2}
\leq \mathbb E_{\vct\theta}\sup_{\|v\|_2=1}\left|\frac{1}{n}\sum_{i=1}^n{v(X_i)} - \mathbb E_{\vct\theta}[v(X)]\right|.
\label{eq:sup-convergence-1}
\end{equation}
Consider the unit $V^m$-dimensional $\ell_2$ ball $\mathbb B_2(V^m) := \{z\in\mathbb R^{V^m}: \|z\|_2\leq 1\}$.
Using standard empirical process theory (e.g., \citep[Lemma 19.36]{van1998asymptotic}, \citep[Theorem 1.1]{talagrand1994sharper}) we have 
\begin{equation}
\mathbb E_{\vct\theta}\sup_{\|v\|_2\leq 1}\sqrt{n}\left|\frac{1}{n}\sum_{i=1}^n{v(X_i)} - \mathbb E_{\vct\theta}[v(X)]\right| \leq C,
\label{eq:sup-convergence-2}
\end{equation}
where $C>0$ is a constant that only depends on $V$ and $m$.
In addition, because $\vct\theta,\vct\theta'\in\Theta_{c_0}$ we know that both $p_{\vct\theta,m}$ and $p_{\vct\theta',m}$ are lower bounded by $c_0^m$ uniformly on $[V]^m$;
hence, for any $\vct\theta'\in\Theta_{c_0}$, using second-order Taylor expansion of the logarithm we have
\begin{align}
\|v_{\vct\theta'}\|_2 
&\leq V^{m/2}\max_{x\in [V]^m}\left|\log\frac{p_{\vct\theta,m}(x)}{p_{\vct\theta',m}(x)}\right| \leq V^{m/2}\max_{x\in[V]^m} 2c_0^{-2m}\big|p_{\vct\theta,m}(x)-p_{\vct\theta',m}(x)\big|\nonumber\\
&\leq 2V^{m/2}c_0^{-2m}\cdot d_{\tv}(p_{\vct\theta,m};p_{\vct\theta',m}) \leq \sqrt{2}V^{m/2}c_0^{-2m}\cdot \sqrt{\kl(p_{\vct\theta,m}\|p_{\vct\theta',m})}.
\label{eq:sup-convergence-3}
\end{align}
Here the last inequality holds by Pinsker's inequality.
Combining Eqs.~(\ref{eq:sup-convergence-1},\ref{eq:sup-convergence-2},\ref{eq:sup-convergence-3}) we complete the proof of Lemma \ref{lem:sup-convergence}.

\subsection{Proof of Proposition \ref{prop:continuous}}

%{\bf [Proof here.]}
By definition, for fixed $c_0,m$ and any two $\vct\theta,\theta'\in\Theta_{c_0}$, we have
\begin{align*}
d_{\tv}(p_{\vct\theta,m};p_{\vct\theta',m})
&= \int_{\mathcal X^m}\big|p_{\vct\theta}(x)-p_{\vct\theta'}(x)\big|\ud\mu_m(x)\\
&\leq V^m\cdot \max_{x\in\mathcal X^m} \big|p_{\vct\theta}(x)-p_{\vct\theta'}(x)\big|\\
&= V^m\cdot \max_{x\in\mathcal X^m}\left|\int_{\Delta^{K-1}} [p_{\vct\theta,h}(x)-p_{\vct\theta',h}(x)]\ud\nu_0(h)\right|\\
&\leq V^m\cdot\max_{x\in\mathcal X^m}\sup_{h\in\Delta^{K-1}}\big|p_{\vct\theta,h}(x)-p_{\vct\theta',h}(x)\big|\\
&= V^m\cdot \max_{x\in\mathcal X^m}\sup_{h\in\Delta^{K-1}}\bigg| \prod_{i=1}^m\left(\sum_{j=1}^Kh_j\theta_j(x_i)\right) - \prod_{i=1}^m\left(\sum_{j=1}^Kh_j\theta_j'(x_i)\right)\bigg|\\
&\leq V^m\cdot \max_{x\in\mathcal X^m}\sup_{h\in\Delta^{K-1}}\sum_{j_1,\cdots,j_m=1}^Kh_{j_1}\cdots h_{j_m} \big|\theta_{j_1}(x_1)\cdots\theta_{j_m}(x_m) - \theta_{j_1}'(x_1)\cdots\theta_{j_m}'(x_m)\big|\\
&\leq V^m\cdot \max_{x\in\mathcal X^m}\sup_{h\in\Delta^{K-1}}\max_{j_1,\cdots,j_m\in[K]} \big|\theta_{j_1}(x_1)\cdots\theta_{j_m}(x_m) - \theta_{j_1}'(x_1)\cdots\theta_{j_m}'(x_m)\big|\\
&\;\;\;\;\cdot \left(\sum_{j_1,\cdots,j_m=1}^K h_{j_1}\cdots h_{j_m}\right)\\
&= V^m\cdot \max_{x\in\mathcal X^m}\sup_{h\in\Delta^{K-1}}\max_{j_1,\cdots,j_m\in[K]} \big|\theta_{j_1}(x_1)\cdots\theta_{j_m}(x_m) - \theta_{j_1}'(x_1)\cdots\theta_{j_m}'(x_m)\big|.
\end{align*} 
Here the last inequality holds because $\sum_{j=1}^Kh_j = 1$.
Furthermore,
because $\theta_j(x),\theta_j'(x)\in(0,1]$, we have
\begin{align*}
&\max_{j_1,\cdots,j_m\in[K]} \big|\theta_{j_1}(x_1)\cdots\theta_{j_m}(x_m) - \theta_{j_1}'(x_1)\cdots\theta_{j_m}'(x_m)\big|\\
&\leq \max_{j_1,\cdots,j_m\in[K]} \sum_{\ell=1}^m \binom{m}{\ell} \left(\max_{j\in[K],x\in[V]} |\theta_j(x)-\theta_j'(x)|\right)^\ell \\
&\leq 2^m \cdot\max_{j\in[K], x\in [V]} |\theta_j(x)-\theta_j'(x)|\\
&\leq 2^m\|\vct\theta'-\vct\theta'\|_2.
\end{align*}

Therefore, we have for any $\vct\theta,\vct\theta'\in\Theta_{c_0}$ that
$$
d_{\tv}(p_{\vct\theta,m};p_{\vct\theta',m}) \leq (2V)^m\cdot \|\vct\theta'-\vct\theta\|_2,
$$
and the proposition is proved, because both $V$ and $m$ are fixed quantities independent of $\vct\theta$ or $\vct\theta'$.

\subsection{Proof of Proposition \ref{prop:pinsker-reverse}}

We prove a more general statement: if $P$ and $Q$ are distributions uniformly lower bounded by a constant $c>0$ on a finite domain $\mathcal D$, 
then there exists a constant $C>0$ depending only on $c$ such that $\kl(P\|Q)\leq C\cdot d_{\tv}^2(P;Q)$.
This implies Proposition \ref{prop:pinsker-reverse} because for any $\vct\theta\in\Theta_{c_0}$, $p_{\vct\theta,m}$ 
is uniformly lower bounded by $c_0^m$ on $\mathcal X^m$.

Let $\mu$ be the counting measure on $\mathcal D$.
Using the definition of KL divergence and second-order Taylor expansion of the logarithm, we have 
\begin{align*}
\kl(P\|Q) 
&= \int_{\mathcal D}P\log\frac{P}{Q}\ud\mu
= \int_{\mathcal D} P\log\left(1+\frac{P-Q}{Q}\right)\ud \mu\\
&\leq \int_{\mathcal D}\frac{P^2}{Q}\ud\mu - 1 + \int_{\mathcal D}\frac{P(P-Q)^2}{2Q^2}\ud\mu\\
&= \int_{\mathcal D}\frac{P^2-Q^2}{Q}\ud\mu + \int_{\mathcal D}\frac{P(P-Q)^2}{2Q^2}\ud\mu\\
&= \int_{\mathcal D}\frac{(P-Q)^2+2PQ-2Q^2}{Q}\ud\mu + \int_{\mathcal D}\frac{P(P-Q)^2}{2Q^2}\ud\mu\\
&= \int_{\mathcal D}{\frac{(P-Q)^2}{Q}\ud\mu} + \int_{\mathcal D}\frac{P(P-Q)^2}{2Q^2}\ud\mu\\
&\leq (1/2c^2+1/c)\cdot \int_{\mathcal D}{(P-Q)^2\ud \mu}.
\end{align*}
On the other hand, $d_{\tv}(P;Q) = \int_{\mathcal D}|P-Q|\ud\mu \geq \sqrt{\int_{\mathcal D}(P-Q)^2\ud\mu}$.
Therefore, $\kl(P\|Q)\leq (1/2c^2+1/c)\cdot d^2_{\tv}(P;Q)$.

\subsection{Proof of Proposition \ref{prop:xi}}
We prove that $\SPAN\{\delta_j\}_{j=1}^2\subseteq\SPAN\{\beta\bar\theta+\gamma\theta_j\}_{j=1}^2$, which would then imply the proposition.
Re-arranging terms in Eq.~(\ref{eq:cancel-intermediate4}) we have
$$
\sum_{j=1,2}\mu_{j\ell}\delta_j = -\sum_{j=1,2}\delta_j(\ell)(\beta\bar\theta+\gamma\theta_j), \;\;\;\;\;\forall \ell\in[V].
$$
Comparing both sides of the above identity it is clear that $\SPAN\{\sum_{j=1,2}\mu_{j\ell}\delta_j\}_{\ell=1}^V \subseteq\SPAN\{\beta\bar\theta+\gamma\theta_j\}_{j=1}^2$.
It remains to prove $\SPAN\{\delta_j\}_{j=1}^2\subseteq\SPAN\{\sum_{j=1,2}\mu_{j\ell}\delta_j\}_{\ell=1}^V$.

Recall that $\mu_{j\ell}=\beta\bar\theta(\ell)+\gamma\theta_j(\ell)$.
Because $\beta\bar\theta+\gamma\theta_1$ and $\beta\bar\theta+\gamma\theta_2$ are linear independent,
we know that $\dim\SPAN\{\beta\bar\theta+\gamma\theta_j\}_{j=1}^2=2$ and hence $\dim\SPAN\{(\mu_{1\ell},\mu_{2\ell})\}_{\ell=1}^V=2$,
because the row rank and the column rank of a matrix are equal.
Thus, for any $(u,v)\in\mathbb R^2$, there exists real coefficients $\{w_\ell\}_{\ell=1}^V$ such that $(u,v) = \sum_{\ell=1}^V{w_\ell(\mu_{1\ell},\mu_{2\ell})}$.
This implies $\SPAN\{\delta_j\}_{j=1}^2\subseteq\SPAN\{\sum_{j=1,2}\mu_{j\ell}\delta_j\}_{\ell=1}^V$, which completes the proof.

\bibliographystyle{apa-good}
\bibliography{refs}

\end{document}